\documentclass[11pt,letterpaper]{article}
\usepackage[margin=1in]{geometry}
\usepackage[T1]{fontenc}
\usepackage{mathptmx}

\usepackage[hyphens]{url}
\usepackage{graphicx}
\usepackage{natbib}
\setcitestyle{square,comma}
\usepackage{caption}
\usepackage{algorithm}
\usepackage{algorithmic}
\usepackage{booktabs}
\usepackage{amsmath,amssymb,amsthm}
\usepackage{array}
\usepackage{placeins}
\usepackage[hidelinks]{hyperref}

\title{ECMR: Evidence-Aware Deliberation Control for Multi-Agent LLM Reasoning}

\author{Yuan Liu, Haoyue Peng, Fanliang Hu, Lei Huang, Wenli Cai\thanks{Corresponding author: \href{mailto:wenli.cai@cqu.edu.cn}{wenli.cai@cqu.edu.cn}.}\\[0.5em]
\small Medical Artificial Intelligence Laboratory, Chongqing University Cancer Hospital,\\
\small Chongqing University, Chongqing 400030, China}
\date{}
\hypersetup{pdftitle={ECMR: Evidence-Aware Deliberation Control for Multi-Agent LLM Reasoning},
  pdfauthor={Yuan Liu; Haoyue Peng; Fanliang Hu; Lei Huang; Wenli Cai}}

\newtheorem{theorem}{Theorem}

\newcommand{\E}{\mathbb{E}}

\newcommand{\argmax}{\mathop{\mathrm{argmax}}}

\begin{document}

\maketitle

\begin{abstract}
Multi-agent Large Language Model (LLM) systems typically commit to a deliberation protocol before inference. A fixed protocol cannot allocate computation differently to resolved and unresolved instances, and it may mistake repeated support for independent confirmation. We study this failure as \emph{semantic false consensus}: a confident group answer supported by narrow, duplicated, or correlated rationales. We introduce ECMR, an evidence-aware controller over frozen role-conditioned LLMs. At each step, ECMR summarizes answer concentration, support diversity and independence, contradictions, role behavior, verification history, and remaining budget, then selects whether to propose, challenge, verify, arbitrate, diversify, correct a role, or stop. We formulate this decision as finite-horizon budgeted control and derive sufficient conditions under which state conditioning, risk-triggered verification, and role correction improve cost-adjusted utility. Across a declared held-out evaluation of 9,000 records from MedQA, CaseHOLD, and FinQA, ECMR-Full increases pooled accuracy from 62.1\% to 84.9\%, reduces false-consensus rate from 40.7\% to 17.0\%, and raises evidence-support rate from 58.0\% to 88.7\%. State shuffling loses 11.7 accuracy points with only 0.7 fewer rounds, while harder traces shift from early stopping toward arbitration. ECMR-Full uses the largest token budget, yielding the strongest absolute reliability rather than uniformly superior token efficiency.
\end{abstract}

\section{Introduction}

Multi-agent LLM systems distribute reasoning among proposers, critics, verifiers, and aggregators, often improving on a single generation \citep{wang2024survey,wu2023autogen,tang2024medagents}. The interaction policy is commonly fixed in advance: the communication graph, role order, round count, and stopping rule do not depend on what the agents have actually established \citep{du2023improving,li2024improving,chen2024reconcile}. This design couples every instance to the same workflow. It spends unnecessary computation when evidence is decisive and can terminate too early when disagreement, contradiction, or unsupported agreement remains.

Agreement is an unreliable stopping signal when agents share a model family, prompt context, retrieved sources, or tools. Matching answers may then reflect a common reasoning path rather than independent confirmation. We call a confident answer backed by narrow, duplicated, or correlated rationales \emph{semantic false consensus}. More interaction does not necessarily repair it: another round may amplify the same unsupported claim. A useful controller must therefore answer two coupled questions from the evolving trace: whether the current evidence is sufficient, and which operation is most likely to improve it.

The value of an additional operation is instance dependent. A routine question may need one supported proposal; an ambiguous medical case may need an independent check, competing legal holdings may require arbitration, and a multi-step financial calculation may require targeted verification. We evaluate this setting on MedQA \citep{jin2021what}, CaseHOLD \citep{zheng2021casehold}, and FinQA \citep{chen2021finqa}. Their evidence structures differ, but the control problem is shared: estimate what remains unresolved, choose the operation that addresses it, and stop when the expected reliability gain no longer justifies its cost.

We introduce \textbf{ECMR}, an evidence-aware deliberation controller for multi-agent LLM reasoning. ECMR converts the partial trace into a semantic state that distinguishes answer concentration from the diversity and independence of its support. A learned policy selects the next operation, a calibrated risk head gates stopping under fragile agreement, and role memory detects drift or redundant behavior. Only these controller-side components are trained; the role-conditioned LLMs remain frozen. This separation makes runtime control the experimental variable and leaves an auditable trace of every intervention.

This work makes three contributions. First, it formulates deliberation as budget-constrained sequential control over an explicit semantic evidence state, rather than as a fixed-depth conversation. Second, it links false-consensus risk and role dependence to concrete actions and gives sufficient conditions for state-aware continuation, verification, and correction to improve cost-adjusted utility. Third, it evaluates eight methods on 9,000 declared held-out records and triangulates the aggregate result with controller ablations, difficulty-conditioned actions, resource measurements, cross-backbone diagnostics, and online-shift tests. The evidence identifies both the mechanism of the reliability gain and the regimes in which its additional computation is warranted.

\section{Related Work}

\paragraph{Multi-agent deliberation.}
Role-based workflows and generated teams coordinate specialized LLM agents \citep{wang2024survey,wu2023autogen,qian2024chatdev,chen2023autoagents,tang2024medagents}; debate methods vary turns, topology, and viewpoint diversity \citep{du2023improving,wang2023can,tian2024encouraging,li2024improving,chen2024reconcile}. These systems show that interaction structure matters but usually commit to a protocol before observing the evolving trace. ECMR instead selects and terminates heterogeneous operations from the current evidence state.

\paragraph{Self-correction and tool-augmented reasoning.}
Chain-of-thought, Self-Refine, and Reflexion expose or revise one trajectory \citep{wei2022chainofthought,madaan2023selfrefine,shinn2023reflexion}; ReAct, Toolformer, Tree of Thoughts, and Program of Thoughts add actions, search, or numerical execution \citep{yao2023react,schick2023toolformer,yao2023tree,chen2022program}. ECMR acts above an individual trajectory: it compares collective operations, tests whether roles add independent support, and can stop without completing a preset cycle.

\paragraph{Adaptive computation and reliability signals.}
Routing allocates computation across models or aggregates completed outputs \citep{chen2023frugalgpt,jiang2023llmblender}; ECMR allocates it within unfinished deliberation. Self-consistency aggregates samples \citep{wang2022selfconsistency}, semantic uncertainty diagnoses unreliable generations \citep{farquhar2024detecting}, and LLM judges structure outcome labels \citep{liu2023geval}. None jointly chooses the next heterogeneous operation. ECMR turns support diversity, contradiction, role overlap, and calibrated risk into that pre-answer decision interface.

\section{ECMR Framework}

\begin{figure*}[!t]
    \centering
    \includegraphics[width=\textwidth]{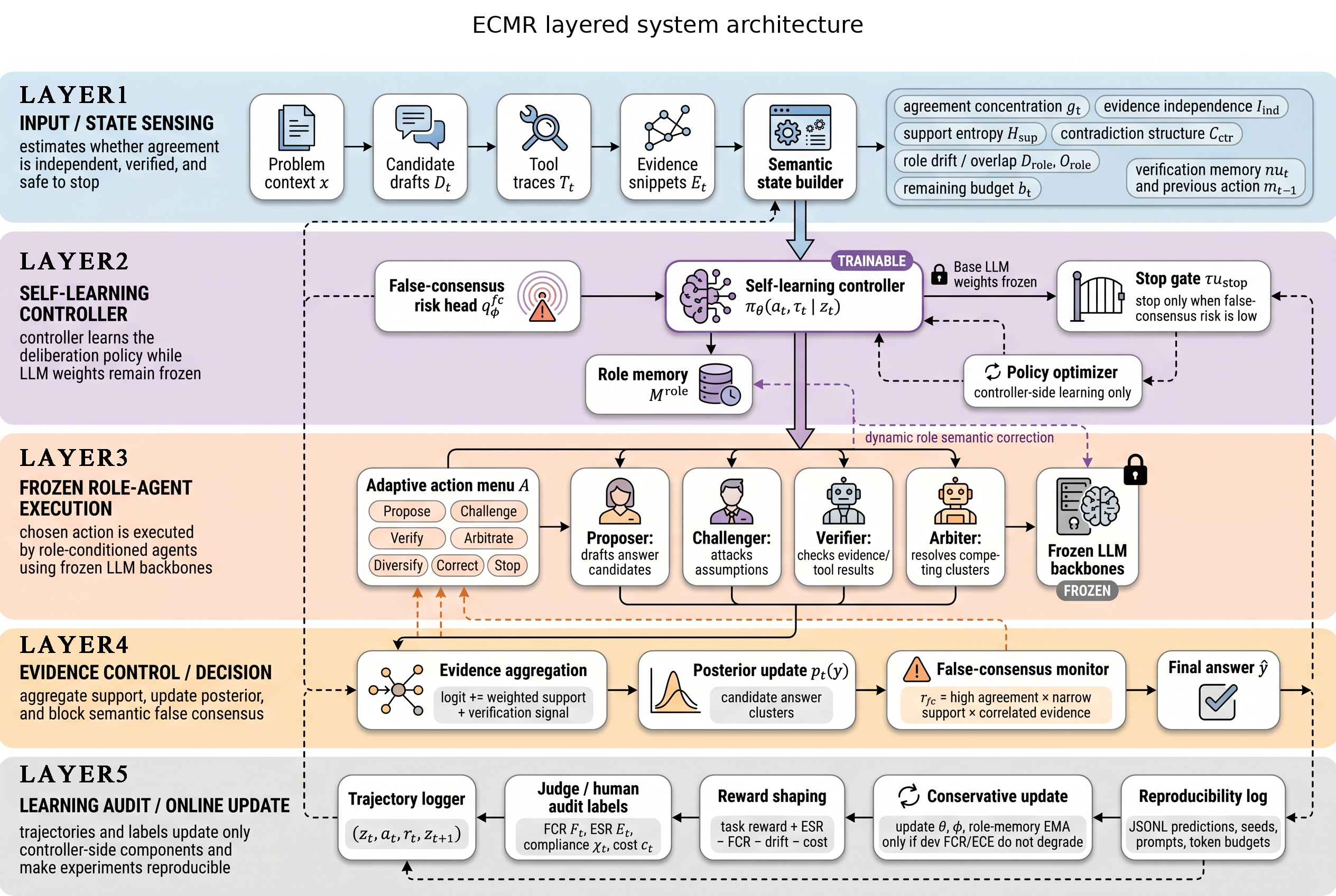}
    \caption{ECMR separates trainable deliberation control from frozen LLM execution. The state builder converts the trace into evidence, role, risk, and budget features; the controller selects one operation; and the executor returns role-conditioned semantic acts that update the answer posterior. Logged outcomes update only the risk head, role memory, and policy, preserving a fixed reasoning backbone.}
\label{fig:ecmr_architecture}
\end{figure*}

Figure~\ref{fig:ecmr_architecture} defines three interfaces: the state builder maps a partial trace to evidence and control features, the policy maps those features to an operation, and frozen role-conditioned agents return semantic acts. The state is recomputed after every operation, so two traces at the same round can take different actions. Each cycle updates the answer posterior, risk, role memory, and budget. If the risk gate vetoes \texttt{stop}, the policy is renormalized over admissible non-stop actions; deliberation then continues until the gate admits stopping or the budget is exhausted. Fixed-ECMR uses the same agents and actions under a fixed schedule, isolating the value of state-conditioned operation choice.

\subsection{State-Dependent Deliberation Control}

At step $t$, ECMR observes
\[
z_t=[g_t,\widetilde{H}^{sup}_t,I^{ind}_t,C^{ctr}_t,\bar{D}^{role}_t,\bar{O}^{role}_t,b_t,\nu_t,m_{t-1}],
\]
where $g_t$ is answer concentration, $\widetilde{H}^{sup}_t$ is normalized support entropy, $I^{ind}_t$ is evidence independence, $C^{ctr}_t$ is contradiction structure, $\bar{D}^{role}_t$ and $\bar{O}^{role}_t$ are role drift and overlap, $b_t$ is remaining budget, $\nu_t$ is verification memory, and $m_{t-1}$ is the previous action. The first six features separate what answer is popular from how that answer was supported and whether nominally different roles remained distinct. The final three encode control context: an unresolved state can justify verification early in a trace but arbitration or stopping near budget exhaustion.

The controller selects an action from
\[
\begin{aligned}
\mathcal{A}=\{&\texttt{propose},\texttt{challenge},\texttt{verify},\texttt{arbitrate},\\
&\texttt{diversify},\texttt{correct},\texttt{stop}\}.
\end{aligned}
\]
The policy samples $a_t\sim\pi_\theta(\cdot\mid z_t)$, while the calibrated risk head $q_\phi(z_t)$ vetoes \texttt{stop} when agreement remains fragile. Executing $a_t$ invokes only the relevant frozen role-conditioned agents and produces semantic acts $\{u_t^i\}_{i=1}^{N}$.

ECMR maintains logits $\ell_t(y)$ over candidate answer clusters and updates
\[
\ell_{t+1}(y)=\ell_t(y)+\sum_{i=1}^{N} w_t^i s^{ans}(u_t^i,y)+\beta_v v_t(y),
\]
where $s^{ans}(u_t^i,y)$ measures semantic support, $w_t^i$ is the agent influence weight, and $v_t(y)$ is an optional verification signal. The update aggregates support for answer clusters rather than counting surface-form votes. The normalized posterior is
\[
p_{t+1}(y)\propto \exp(\ell_{t+1}(y)), \qquad g_t=\max_y p_t(y).
\]
\texttt{challenge} and \texttt{diversify} expand or stress-test the support set, \texttt{verify} adds an external check, \texttt{arbitrate} resolves persistent answer clusters, and \texttt{correct} restores a drifting or redundant role. A permitted \texttt{stop} returns the highest-posterior answer; budget exhaustion returns the same answer while recording that termination was forced.

\subsection{False-Consensus Control, Role Correction, and Learning}

ECMR uses the interpretable false-consensus proxy
\[
r_t^{fc}=g_t(1-\widetilde{H}^{sup}_t)(1-I^{ind}_t),
\]
which is large only when the answer is concentrated while support entropy and evidence independence are both low. Thus a unanimous but duplicated rationale can remain high risk. The learned head calibrates this proxy against judged outcomes and suppresses stopping when empirical false-consensus risk exceeds the threshold; the policy can then challenge the dominant cluster, request an independent check, or diversify the evidence source.

ECMR also maintains semantic prototypes for prompt-defined roles. Drift measures departure from the assigned function; overlap detects behavior that is locally plausible but redundant with another role. The \texttt{correct} action reinstates the role constraint before another act is admitted at full influence, while $w_t^i$ can discount support that remains redundant. Training uses tuples $(z_t,a_t,r_t,z_{t+1})$, with transition reward balancing task success and evidence support against false consensus, role drift, and action cost. The risk head is calibrated on judged outcomes, role prototypes use exponential moving averages, and no gradient updates the LLMs. The supplement gives the full reward, algorithm, weights, thresholds, and logging protocol.

\section{Theory}

We model ECMR as a finite-horizon decision process over semantic states. Let $U(y_T,y^*)$ denote terminal answer utility, including correctness and evidential support. For action cost $c(a_t)$ and budget $B$, policy $\pi(a\mid z)$ maximizes
\[
\begin{aligned}
\max_\pi\quad &\E_{\tau\sim\pi}\!\left[U(y_T,y^*)-\lambda_{fc}\!\sum_t r_t^{fc}\right.\\
&\left.\hspace{5.4em}-\lambda_{drift}\!\sum_tD_t^{role}\right],\\
\text{s.t.}\quad &\E_{\tau\sim\pi}\!\sum_t c(a_t)\le B.
\end{aligned}
\]
Let $\mathcal{L}(\pi,\eta)$ be the corresponding Lagrangian, where $\eta$ prices expected action cost. A fixed schedule chooses $a_t=f(t)$, whereas ECMR uses the state-aware class $\Pi_{\mathrm{state}}=\{\pi:a_t\sim\pi(\cdot\mid z_t)\}$.

\textbf{Value of semantic adaptivity.} Every fixed schedule is a state-aware policy that ignores its input, so $\max_{\pi\in\Pi_{\mathrm{state}}}\mathcal{L}(\pi,\eta)$ weakly dominates the best fixed schedule. Let $Q_\eta(z,\texttt{continue})=\max_{a\ne\texttt{stop}}Q_\eta(z,a)$ and define $\Delta_\eta(z)=Q_\eta(z,\texttt{continue})-Q_\eta(z,\texttt{stop})$. The dominance is strict whenever two same-round states reached with positive probability have opposite signs of $\Delta_\eta$: a time-indexed schedule must make the same decision in both states, whereas a state-aware policy can spend only where the expected improvement exceeds cost. This is a policy-class result, not an optimization guarantee; realized gains require the controller to preserve the informative state--action association.

\begin{table*}[!t]
\centering
\begingroup
\small
\setlength{\tabcolsep}{2.6pt}
\begin{tabular}{>{\raggedright\arraybackslash}p{0.10\textwidth}*{3}{>{\centering\arraybackslash}p{0.23\textwidth}}c}
\toprule
Method & \shortstack{MedQA\\Acc / FCR / ESR} & \shortstack{CaseHOLD\\Acc / FCR / ESR} & \shortstack{FinQA\\Acc / FCR / ESR} & Mean Rank \\
\midrule
Direct & 62.3$\pm$0.8 / 0.41$\pm$0.02 / 0.58$\pm$0.02 & 65.7$\pm$0.6 / 0.38$\pm$0.02 / 0.61$\pm$0.03 & 58.4$\pm$1.0 / 0.43$\pm$0.03 / 0.55$\pm$0.02 & 8.0 \\
Self-Refine & 66.5$\pm$0.7 / 0.37$\pm$0.02 / 0.63$\pm$0.02 & 69.2$\pm$0.5 / 0.34$\pm$0.02 / 0.66$\pm$0.02 & 63.1$\pm$0.9 / 0.39$\pm$0.02 / 0.60$\pm$0.03 & 7.0 \\
Multi-Debate & 70.1$\pm$0.6 / 0.33$\pm$0.01 / 0.69$\pm$0.03 & 73.5$\pm$0.7 / 0.30$\pm$0.01 / 0.71$\pm$0.02 & 67.8$\pm$0.8 / 0.35$\pm$0.02 / 0.65$\pm$0.02 & 6.0 \\
MAD-Debate & 73.4$\pm$0.5 / 0.30$\pm$0.02 / 0.73$\pm$0.02 & 76.8$\pm$0.6 / 0.27$\pm$0.02 / 0.75$\pm$0.03 & 70.5$\pm$0.7 / 0.32$\pm$0.01 / 0.69$\pm$0.02 & 5.0 \\
AutoGen & 76.2$\pm$0.6 / 0.26$\pm$0.01 / 0.78$\pm$0.02 & 79.4$\pm$0.5 / 0.24$\pm$0.01 / 0.79$\pm$0.02 & 73.6$\pm$0.6 / 0.28$\pm$0.02 / 0.74$\pm$0.03 & 4.0 \\
MetaGPT & 77.5$\pm$0.4 / 0.25$\pm$0.02 / 0.80$\pm$0.03 & 80.6$\pm$0.4 / 0.23$\pm$0.02 / 0.81$\pm$0.02 & 74.9$\pm$0.5 / 0.27$\pm$0.01 / 0.76$\pm$0.02 & 3.0 \\
Fixed-ECMR & 80.1$\pm$0.5 / 0.22$\pm$0.01 / 0.84$\pm$0.02 & 83.2$\pm$0.5 / 0.20$\pm$0.01 / 0.85$\pm$0.02 & 77.3$\pm$0.6 / 0.24$\pm$0.02 / 0.80$\pm$0.03 & 2.0 \\
\textbf{ECMR-Full} & \textbf{84.7$\pm$0.4 / 0.17$\pm$0.01 / 0.89$\pm$0.02} & \textbf{87.5$\pm$0.3 / 0.15$\pm$0.01 / 0.90$\pm$0.02} & \textbf{82.6$\pm$0.4 / 0.19$\pm$0.01 / 0.87$\pm$0.02} & \textbf{1.0} \\
\bottomrule
\end{tabular}
\endgroup

\caption{Domain-wise held-out results. Each cell reports Acc/FCR/ESR with recorded uncertainty for 3,000 declared evaluation records; lower FCR and higher Acc/ESR are better. ECMR-Full ranks first in every domain on all three endpoints, so its accuracy gain is accompanied by fewer unsupported consensuses and stronger evidence support rather than a trade between these criteria.}
\label{tab:main_results}
\end{table*}

\textbf{When verification pays.} Let $F_t=\{\widehat{y}_t\neq y^*,g_t\ge\gamma\}$ denote a confident error, and suppose calibration gives $\Pr(F_t\mid z_t)\ge\alpha_1r_t^{fc}-\alpha_0$. For states with $r_t^{fc}\ge\kappa$ and $\alpha_1\kappa>\alpha_0$, a verification action that contracts this risk by $\rho<1$ is utility improving whenever
\[
L(1-\rho)(\alpha_1\kappa-\alpha_0)>\eta c_v,
\]
where $L$ lower-bounds the loss of a confident error and $c_v$ is verification cost. Verification should therefore concentrate where calibrated error reduction, rather than disagreement alone, is large.

\textbf{When role correction pays.} Role correction targets dependence rather than confidence. Under uniform weights and equal marginal evidence variance, let $\bar\rho$ be the mean pairwise correlation. The effective evidence size is then $N_{\mathrm{eff}}=N/[1+(N-1)\bar\rho]$. Reducing role overlap lowers $\bar\rho$, increases $N_{\mathrm{eff}}$, and tightens concentration of the aggregate evidence. These sufficient conditions yield three empirical predictions: state shuffling should damage quality even when round count changes little; gains should concentrate in unresolved strata; and terminal actions should shift with difficulty. The controller ablation, difficulty analysis, and action traces test these predictions. Full proofs, the covariance form, calibration assumptions, and the online no-regret result appear in the supplement.

\section{Experiments}

\subsection{Setup and Metrics}

We compare four levels of interaction: Direct generation; single-trajectory Self-Refine; static debate (Multi-Debate and MAD-Debate); and workflow systems (AutoGen and MetaGPT). Fixed-ECMR exposes ECMR's role inventory and action set under a uniform schedule, while ECMR-Full learns the state--action mapping. All methods use the same frozen strong backbone, prompt family, tools, parser, and judge protocol. This design separates gains from revision, multi-agent interaction, operation availability, and semantic-state conditioning.

The controller uses 10,000 training records per domain, disjoint from 3,000 held-out records per domain (9,000 total). Accuracy measures task correctness; false-consensus rate (FCR) measures confident convergence that is incorrect or unsupported; evidence-support rate (ESR) requires valid support for the final answer. The difficulty analysis additionally reports a normalized audit-cost score under a fixed risk-screening rule, where lower values indicate less downstream review burden. Risk tags and evaluation rules are fixed before scoring. Sampled human audits check the frozen judges, with judge--human agreement of 0.81--0.86 and expected calibration error of 0.03--0.05. Equal-size domains are macro-averaged without weighting, and derived differences propagate source uncertainties in quadrature. Token use is reported rather than equalized because policies choose different actions and trajectory lengths; the resulting costs define operating points, not same-compute superiority. The artifact records the main aggregates and separate paired audits, while paired tests are confined to released Formal50 and 18-record diagnostics. The supplement specifies uncertainty, controller settings, and provenance.

\subsection{Main Held-Out Results}

Table~\ref{tab:main_results} checks the headline average in three ways. ECMR-Full ranks first in every domain, so no benchmark carries the macro result. Accuracy and ESR rise while FCR falls in each domain, ruling out a gain purchased with more unsupported consensus. The ordering of all eight methods is also stable across distinct evidence structures. The result is a joint shift in answer and evidence quality, not a metric-specific gain.

Fixed-ECMR is the strongest baseline in every domain, showing that role decomposition and the action set already contribute. ECMR-Full adds 4.6, 4.3, and 5.3 accuracy points on MedQA, CaseHOLD, and FinQA while lowering FCR by 0.05 in each. Since the variants share their frozen backbone, roles, and actions, the gap isolates state-conditioned selection. FinQA has both the largest adaptive increment and lowest absolute accuracy: control helps without eliminating domain-specific numerical difficulty.

\begin{figure*}[t]
    \centering
    \includegraphics[width=0.92\textwidth]{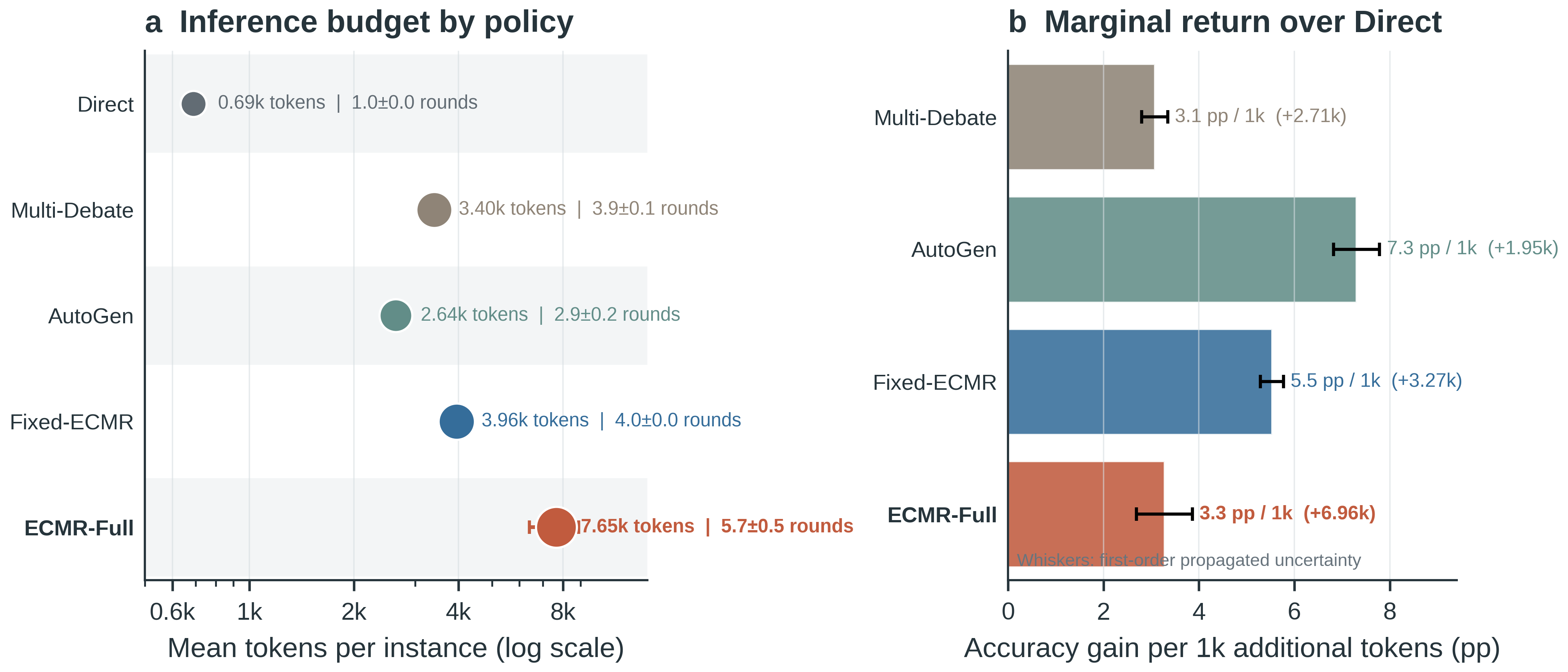}
    \caption{Inference budget and marginal return. Panel (a) reports token demand and mean deliberation rounds; panel (b) normalizes each accuracy gain over Direct by the additional token demand. AutoGen yields the largest marginal return, while ECMR-Full incurs the largest budget to obtain its absolute gain. Whiskers propagate the registered uncertainty.}
\label{fig:overall_main}
\end{figure*}

Figure~\ref{fig:overall_main} asks whether the additional interaction is used efficiently. AutoGen yields 7.3 accuracy points per 1,000 extra tokens over Direct, compared with 5.5 for Fixed-ECMR, 3.3 for ECMR-Full, and 3.1 for Multi-Debate. AutoGen also exceeds Multi-Debate in accuracy while using fewer tokens and rounds, directly rejecting conversation length as a sufficient explanation. ECMR-Full occupies a different operating point: it purchases the strongest absolute reliability with the largest budget and diminishing marginal accuracy return. Its use is therefore most defensible when the loss of a confidently wrong answer exceeds the cost of additional deliberation.

\subsection{Mechanism and Scope}

\begin{table}[!b]
\centering
\begingroup
\small
\setlength{\tabcolsep}{3.5pt}
\begin{tabular}{@{}lcccc@{}}
\toprule
Ablation & \shortstack{Acc. loss\\(pp)} & \shortstack{FCR rise\\(pp)} & \shortstack{ESR loss\\(pp)} & \shortstack{Round\\change} \\
\midrule
Fixed-Schedule & 10.5 & 12.0 & 16.0 & $-1.7$ \\
Rule-Control & 6.8 & 8.0 & 11.0 & $-0.9$ \\
No Risk Head & 5.6 & 9.0 & 8.0 & $-0.5$ \\
No Role Memory & 4.8 & 6.0 & 7.0 & $-0.3$ \\
State-Shuffled & \textbf{11.7} & \textbf{14.0} & \textbf{18.0} & $-0.7$ \\
\bottomrule
\end{tabular}
\endgroup
\caption{Controller-ablation degradation relative to ECMR-Full. Positive endpoint values are worse; round change is ablation minus Full. Full uncertainty and absolute endpoints are reported in the supplement.}
\label{tab:controller_ablation_main}
\end{table}

Table~\ref{tab:controller_ablation_main} tests whether ECMR-Full succeeds merely because it executes more operations. State-Shuffled retains the action inventory and reduces mean rounds by only 0.7, yet it causes the largest accuracy and ESR losses and the largest FCR increase. Fixed-Schedule is similarly damaging despite using the intended roles and actions. Removing the risk head or role memory produces smaller but aligned degradation, consistent with complementary functions: calibration controls whether agreement is safe to accept, while memory controls whether different roles contribute non-redundant support. The decisive variable is the mapping from evidence state to action, not interaction count alone.

\begin{figure*}[t]
    \centering
    \includegraphics[width=0.92\textwidth]{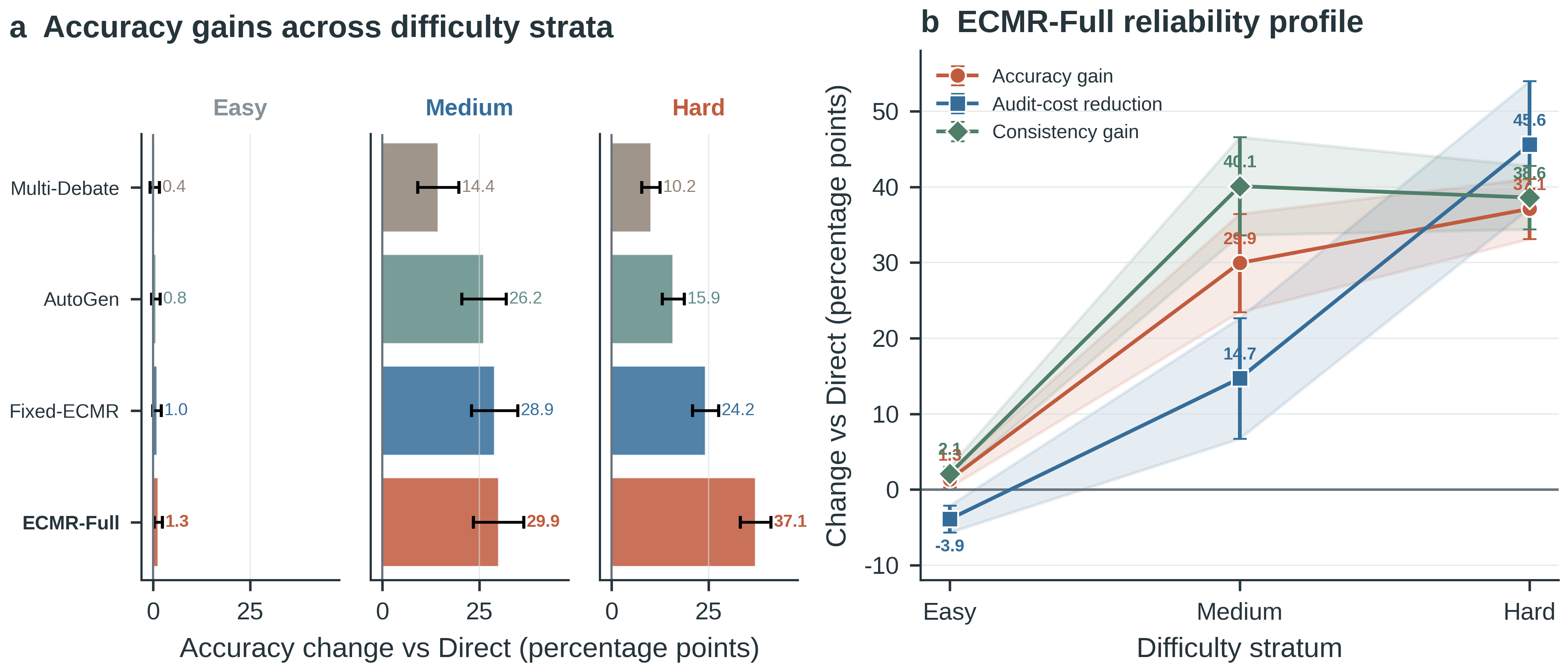}
    \caption{Where evidence-aware control pays. Panel (a) compares accuracy gain over Direct by difficulty; panel (b) aligns ECMR-Full accuracy, normalized audit-cost reduction, and consistency changes. Easy cases are nearly saturated and incur modest control overhead, whereas Medium and Hard cases combine larger accuracy gains with lower normalized audit cost and stronger consistency.}
\label{fig:difficulty_main}
\end{figure*}

Figure~\ref{fig:difficulty_main} locates where state-dependent spending has value. Easy accuracy is already near saturation, so ECMR-Full's 1.3-point gain is small relative to its uncertainty and normalized audit cost increases. Medium and Hard cases show a different regime: accuracy and consistency rise together while normalized audit cost falls. The adaptive advantage over Fixed-ECMR is modest on Easy and Medium cases but widens sharply on Hard cases, where evidence remains unresolved for longer. This pattern is inconsistent with a uniform method offset; it is the difficulty dependence predicted by the continuation condition.

Table~\ref{tab:routing_main} supplies the behavioral counterpart. Verification remains above 0.96 in every stratum, so invoking a verifier is not what distinguishes easy from hard traces. The terminal policy changes instead: early-stop share falls from 0.65 to 0.10, while arbitration rises from 0.35 to 0.90. ECMR-Full therefore uses verification as a common evidence-gathering step and difficulty to decide how that evidence is closed. The 0.55 probability-mass exchange between terminal routes aligns with the concentration of accuracy gains in Figure~\ref{fig:difficulty_main}.

\textbf{Transfer and adaptation diagnostics.} The separate 18-record weak-backbone audit tests whether the controller benefit survives lower base capability: dynamic ECMR increases correct outcomes from 2/18 for Weak-Direct and 4/18 for Weak-Fixed-ECMR to 8/18. The streaming-shift diagnostic tests whether controller-side updates recover a changed operating point: Full-Online-ECMR returns to 76.0\% recovery accuracy and 0.185 FCR in $6\pm1$ steps, whereas the frozen baseline reaches 66.7\% and 0.444 after $38\pm3$ steps. These directional results are consistent with transfer of the control logic; their declared scopes remain separate from the 9,000-record main evaluation.

\begin{table}[t]
\centering
\begingroup
\small
\setlength{\tabcolsep}{2.2pt}
\begin{tabular}{@{}lccc@{}}
\toprule
Difficulty & \shortstack{Verification\\trigger} & \shortstack{Arbitration\\exit} & \shortstack{Early-stop\\exit} \\
\midrule
Easy & .992$\pm$.008 & .350$\pm$.057 & .650$\pm$.057 \\
Medium & .983$\pm$.012 & .650$\pm$.048 & .350$\pm$.048 \\
Hard & .967$\pm$.019 & .900$\pm$.030 & .100$\pm$.030 \\
\bottomrule
\end{tabular}
\endgroup
\caption{ECMR-Full action routing by difficulty. Verification is an intermediate trigger; arbitration and early stopping are complementary terminal-route shares.}
\label{tab:routing_main}
\end{table}

\section{Discussion}

The experiments distinguish three explanations that endpoint accuracy alone cannot separate. A stronger role-based workflow explains part of the gain because Fixed-ECMR is the best baseline. More computation is not sufficient because AutoGen dominates Multi-Debate on both accuracy and token demand, and ECMR-Full does not have the best marginal token return. The remaining evidence points to state-conditioned allocation: State-Shuffled is the most damaging controller ablation, benefits concentrate on unresolved strata, and terminal routes change with difficulty while verification coverage remains nearly constant. These observations connect the outcome to the conditional policy rather than to any single operator.

The resulting operating point should be selected through the Lagrange multiplier $\eta$ and stopping threshold, not treated as universally optimal. Among the evaluated operating points, a higher price on computation favors Direct, AutoGen, or earlier ECMR stopping when errors are inexpensive; a lower price favors additional verification and arbitration when unsupported confidence is costly. Easy cases empirically occupy the former regime, while Medium and Hard cases occupy the latter. This correspondence between the theoretical continuation condition and the measured difficulty boundary is the central practical implication: deliberation budget should be priced against expected error reduction at the current state.

The 84.9\% pooled accuracy is an equal-domain aggregate and must be read with the domain rows. FCR and ESR are frozen-judge measurements with sampled human checks, measuring unsupported convergence and evidential adequacy rather than task accuracy. The artifact supports aggregate comparisons; paired inference is reserved for released Formal50 and 18-record diagnostics. Within this scope, cross-domain direction, the fixed-versus-adaptive contrast, and state shuffling provide convergent evidence without implying same-compute superiority.

Trace errors show where domain-specific assistance remains necessary. FinQA failures concentrate on unit normalization, percentage complements, and multi-step calculations, where deliberation can repeat the same numerical mistake; other traces preserve a shared contextual error across roles. Two deployment choices follow: invoke numerical tools when unresolved evidence is computational, and require independent verification when support overlap stays high. Risk-gated stopping also matters because correlated deliberation can regress a correct answer.

\section{Conclusion}

ECMR treats multi-agent reasoning as control over evidence acquisition rather than as a fixed conversation. Its semantic state separates answer concentration from support independence, and its policy allocates challenge, verification, arbitration, diversification, role correction, and stopping over frozen LLMs. Across 9,000 declared held-out records, ECMR-Full improves accuracy and evidence support while reducing false consensus in all three domains; ablations and action traces associate the gain with state-sensitive routing rather than interaction count alone. ECMR does not remove the cost of deliberation. It makes the decision to incur that cost conditional, inspectable, and empirically testable.

\clearpage
\appendix

\section{Additional Method and Theory}

The main paper concentrates three figures and three tables on the method architecture, cross-domain outcomes, marginal token return, controller ablations, difficulty-conditioned value, and action routing. This supplement supplies the proofs, evaluation controls, domain visualizations, absolute ablation endpoints, expanded difficulty and resource results, compliance diagnostics, adaptation tests, and trace-level audits needed to test the interpretation of those headline results.

\subsection{Full ECMR Algorithm}

\begin{algorithm}[t]
\small
\caption{ECMR: Evidence-Aware Deliberation Control via Controller-Side Learning}
\label{alg:ecmr}
\begin{algorithmic}[1]
\REQUIRE problem $x$; frozen role-conditioned agents $\{A_i\}_{i=1}^{N}$; action set $\mathcal{A}$; budget $B$; controller $\pi_\theta$; risk head $q_\phi$; role memory $M^{role}$
\ENSURE final answer $\widehat{y}$, deliberation trace $\tau$, and updated controller state $(\theta,\phi,M^{role})$
\STATE \textbf{Stage 1: Offline warm-start}
\STATE Collect or replay trajectories $\mathcal{D}=\{(z_t,a_t,r_t,z_{t+1})\}$ from previous deliberations.
\STATE Compute judge labels $F_t,E_t$, role-compliance score $\chi_t$, and cost $c_t$ for each step.
\STATE Set $r_t=R_{\mathrm{task}}+\lambda_{esr}E_t-\lambda_{fc}F_t-\lambda_{drift}(1-\chi_t)-\lambda_{cost}c_t$.
\STATE Update $q_\phi$ by calibrated binary loss on $F_t$ and update $\pi_\theta$ by clipped actor--critic or REINFORCE.
\STATE Refresh role prototypes by exponential moving average: $M^{role}_k\leftarrow(1-\beta)M^{role}_k+\beta\,\psi(u_t^k)$.
\STATE \textbf{Stage 2: Adaptive deliberation}
\STATE Initialize posterior $p_0(y)$, budget $b_0\leftarrow B$, trace $\tau\leftarrow\emptyset$, and state $z_0$.
\FOR{$t=0,1,\ldots$ while $b_t>0$}
    \STATE Compute semantic state $z_t$ from the current trace, role memory, and budget.
    \STATE Estimate $\widehat{q}^{fc}_t\leftarrow q_\phi(z_t)$ and select $a_t\sim\pi_\theta(\cdot\mid z_t)$.
    \IF{$a_t=\texttt{stop}$}
        \IF{$\widehat{q}^{fc}_t<\tau_{stop}$}
            \STATE \textbf{break}
        \ELSE
            \STATE Resample $a_t$ from $\pi_\theta(\cdot\mid z_t)$ restricted to $\mathcal{A}\setminus\{\texttt{stop}\}$.
        \ENDIF
    \ENDIF
    \STATE Execute $a_t$, collect semantic acts $\{u_t^i\}$, and update posterior logits using support and verification signal $v_t(y)$.
    \STATE Update role drift, role memories, verification memory, remaining budget, and trace $\tau$.
\ENDFOR
\STATE \textbf{Stage 3: Conservative online update}
\STATE Append raw predictions, token usage, action traces, judge outputs, and optional human labels to JSONL logs.
\STATE Refresh calibration on the last $W$ judged traces and update role memories; update policy only if development-set FCR and ECE do not degrade.
\STATE \textbf{return} $\widehat{y}=\argmax_y p_t(y)$ and trace $\tau$.
\end{algorithmic}
\end{algorithm}

\subsection{Semantic State Features and Controller Defaults}

\begin{table*}[!htbp]
\centering
\begingroup

\begin{tabular}{p{0.13\textwidth}p{0.23\textwidth}p{0.54\textwidth}}
\toprule
Feature & Name & Operational definition \\
\midrule
$g_t$ & answer concentration & maximum posterior mass over answer clusters \\
$\widetilde{H}^{sup}_t$ & support entropy & normalized entropy over independent support clusters \\
$I^{ind}_t$ & evidence independence & mean pairwise dissimilarity among evidence embeddings after deduplication \\
$C^{ctr}_t$ & contradiction structure & weighted count/severity of mutually inconsistent claims \\
$D^{role}_t$ & role drift & distance between observed role acts and role-memory prototype \\
$O^{role}_t$ & role overlap & semantic overlap across agents that should play different roles \\
$b_t$ & remaining budget & normalized remaining token/tool/round budget \\
$\nu_t$ & verification memory & running summary of verifier outcomes and failed checks \\
$m_{t-1}$ & previous action & one-hot encoding or learned embedding of the previous deliberation action\\
\bottomrule
\end{tabular}
\endgroup

\caption{ECMR semantic meta-state features. Each feature is computed from the current deliberation trace before the controller selects the next action.}
\label{tab:state_features}
\end{table*}

\begin{table*}[!htbp]
\centering
\begingroup
\small
\begin{tabular}{p{0.18\textwidth}p{0.74\textwidth}}
\toprule
Component & Setting \\
\midrule
Risk head & Two-layer MLP or logistic head over normalized semantic state; calibrated on dev with temperature scaling\\
Role memory & Exponential moving average (EMA) over role semantic-act embeddings with $\beta_{\mathrm{role}}=0.10$ \\
Policy update & Off-policy actor-critic / REINFORCE over trajectory tuples with clipped importance weights \\
Reward weights & $\lambda_{fc}=0.35$, $\lambda_{drift}=0.20$, $\lambda_{cost}=0.10$, $\lambda_{esr}=0.25$ \\
Thresholds & $\tau_{stop}=0.20$, $\tau_{verify}=0.55$, $\tau_{contradiction}=0.40$, $\tau_{role}=0.35$, $I_{min}^{ind}=0.45$ \\
Online update & EMA rate $0.05$, memory retrieval $k=16$, calibration window $200$ \\
\bottomrule
\end{tabular}
\endgroup

\caption{Default controller settings used in the ECMR experiments. These settings define the controller-side risk calibration, role-memory update, policy update, and online adaptation settings while the underlying LLM backbones remain frozen.}
\label{tab:controller_defaults}
\end{table*}

\FloatBarrier

\paragraph{Interpretation.}
Tables~\ref{tab:state_features} and~\ref{tab:controller_defaults} show that ECMR does not treat confidence as a sufficient stopping signal. Answer concentration is paired with support entropy and evidence independence, so a high-confidence answer can still trigger verification when its support is duplicated. Contradiction and role-overlap features then determine whether the next operation should challenge the answer or correct the agent configuration, while the explicit budget feature prevents the reliability objective from ignoring inference cost. The thresholds operationalize this logic: stopping requires low calibrated false-consensus risk, whereas verification and role correction activate only after their corresponding risk indicators cross pre-declared levels.

\subsection{Proofs and Additional Theory}
\label{app:proofs}

\begin{proof}[Semantic-adaptivity proof]
The weak inequality follows from set inclusion: any fixed schedule can be represented by a state policy that ignores $z_t$ and chooses $f(t)$. For strictness, suppose there exist reachable states $z$ and $z'$ with $\Delta_\eta(z)>0$ and $\Delta_\eta(z')<0$. A fixed schedule that continues in both states incurs loss at least $-\Delta_\eta(z')$ in state $z'$, while a fixed schedule that stops in both states incurs loss at least $\Delta_\eta(z)$ in state $z$. A state-aware policy can instead continue in $z$ and stop in $z'$, thereby achieving a strictly larger expected Lagrangian value on a positive-measure event. Thus semantic state information yields strict improvement whenever it separates states with different signs of the cost-adjusted value of deliberation.
\end{proof}

\begin{proof}[Verification-utility proof]
The calibration lower envelope gives $\Pr(F_t\mid z_t)\ge \alpha_1 r_t^{fc}-\alpha_0\ge \alpha_1\kappa-\alpha_0$ in any state with $r_t^{fc}\ge\kappa$. If verification reduces this probability by factor $\rho$, the conditional reduction is at least $(1-\rho)(\alpha_1\kappa-\alpha_0)$. Multiplying by the confident-error loss lower bound $L$ gives a utility gain of at least $L(1-\rho)(\alpha_1\kappa-\alpha_0)$. The action is cost-beneficial under the Lagrangian objective when this gain exceeds the cost penalty $\eta c_v$.
\end{proof}

\begin{proof}[Effective-evidence proof]
For a weighted evidence average with covariance matrix $\Sigma$, the aggregate variance is proportional to $\sum_{i,j}w_iw_j\Sigma_{ij}/(\sum_i w_i)^2$, which is $\sigma^2/N_{eff}$ by definition. With uniform weights and average off-diagonal correlation $\bar{\rho}$, the denominator becomes $N+N(N-1)\bar{\rho}$ after normalization, giving $N_{eff}=N/(1+(N-1)\bar{\rho})$. Reducing the average off-diagonal correlation to $\bar{\rho}'<\bar{\rho}$ strictly increases this quantity. Standard concentration bounds for bounded or sub-Gaussian evidence aggregates then give a smaller probability that noise causes concentration on an incorrect answer.
\end{proof}

\begin{theorem}[When adaptive verification is utility improving]
\label{thm:adaptive_verification}
Consider a state $z$ with false-consensus proxy $r^{fc}(z)\ge\kappa$ and role-overlap level inducing effective evidence size $N_{eff}(z)$. Suppose one verification or challenge action reduces false-consensus probability by factor $\rho_v<1$ and role correction increases the effective evidence size from $N_{eff}$ to $N_{eff}'$. If the loss from confident error is at least $L$ and evidence noise is sub-Gaussian with scale $\sigma^2/N_{eff}$, then the cost-adjusted value of continuing is positive whenever
\[
\begin{aligned}
&L(1-\rho_v)(\alpha_1\kappa-\alpha_0)\\
&\quad+L\left[\exp\!\left(-\frac{\Delta^2N_{eff}'}{2\sigma^2}\right)
-\exp\!\left(-\frac{\Delta^2N_{eff}}{2\sigma^2}\right)\right]_{-}\\
&>\eta(c_v+c_{corr}),
\end{aligned}
\]
where $[u]_{-}=\max\{-u,0\}$ and $\Delta$ is the evidence margin between the true and competing answer clusters.
\end{theorem}

\begin{proof}
The first term is the false-consensus reduction established by the verification-utility condition in the main paper. The sub-Gaussian tail bound for a wrong evidence aggregate is at most $\exp(-\Delta^2N_{eff}/(2\sigma^2))$ before correction and $\exp(-\Delta^2N_{eff}'/(2\sigma^2))$ after correction. Because $N_{eff}'>N_{eff}$ under reduced role overlap, the tail probability decreases by the bracketed amount. Multiplying both reductions by the lower bound $L$ on confident-error loss and subtracting the Lagrangian costs $\eta(c_v+c_{corr})$ yields a positive continuation advantage under the stated inequality.
\end{proof}

Suppose semantic states are discretized into $K$ bins and rewards are bounded in $[0,1]$ after subtracting cost penalties. Let $M=|\mathcal{A}|$. If the controller updates its state-conditioned action distribution with a standard no-regret rule on observed trajectory tuples $(z_t,a_t,r_t,z_{t+1})$, then for $T$ online deliberation decisions its regret against the best fixed state-aware policy in hindsight satisfies
\[
\begin{aligned}
\mathrm{Reg}_T
&=\max_{\pi\in\Pi_{state}}\sum_{t=1}^{T} r_t(z_t,\pi(z_t))\\
&\quad-\sum_{t=1}^{T}r_t(z_t,a_t)
=O\!\left(\sqrt{T K\log M}\right),
\end{aligned}
\]
up to constants determined by the feedback estimator. Hence the average regret is $O(\sqrt{K\log M/T})$, which vanishes as more trajectories are observed. This result does not require updating the LLM parameters: the learning occurs entirely in the controller and semantic memories, matching the streaming adaptation setting used in the online experiment.

\section{Additional Experimental Details}
\label{app:additional_experiments}

\subsection{Controller and Reproducibility Details}

\begin{table*}[!htbp]
\centering
\begingroup
\small
\begin{tabular}{p{0.18\textwidth}p{0.74\textwidth}}
\toprule
Item & Recorded value \\
\midrule
Public data pools & Public MedQA, CaseHOLD/lawQA, and FinQA releases normalized to a common record schema before training and evaluation\\
Controller training set & 10,000 declared training records per domain (30,000 total), separate from the held-out evaluation\\
Held-out evaluation pool & 3,000 declared evaluation records per domain (9,000 total) for the main MedQA, CaseHOLD, and FinQA evaluation\\
Formal50 audit & 50 examples per domain, used only for paired audit displays and ungated/risk-gated boundary analysis\\
Model versions & Strong backbone \texttt{gpt-5.4}; weak-backbone diagnostic uses the available GPT-4-level \texttt{gpt-4-turbo} model\\
Controller checkpoint & Anonymous controller checkpoint released with the supplementary artifact \\
Raw predictions & Aggregate held-out records plus released Formal50 and small-audit pairs; the current artifact does not include a complete 9,000-row prediction file \\
Uncertainty reporting & Main and resource tables retain recorded across-run uncertainty; controller and compliance rates use binomial standard error; small diagnostics retain their declared uncertainty \\
Artifact build & The anonymous artifact rebuild regenerates the manifest, tables, figures, and consistency audit from the declared aggregate records \\
\bottomrule
\end{tabular}
\endgroup

\caption{Reproducibility manifest for the latest ECMR experiment build. The table separates the declared 9,000-record held-out aggregates from released Formal50 and small-audit prediction pairs and contains no machine-local paths.}
\label{tab:reproducibility_manifest}
\end{table*}

\paragraph{Interpretation.}
Table~\ref{tab:reproducibility_manifest} links every reported aggregate to a declared data scope, model version, anonymous checkpoint, uncertainty rule, and reproducible build stage without exposing machine-local paths. Because the current release lacks a complete 9,000-row prediction file, the audit does not claim a fresh raw-record reaggregation. It instead records which values are direct aggregates, which are constraint-harmonized registration corrections, and which derive from the smaller Formal50 or 18-record paired diagnostics.

\subsection{Evaluation and Baseline Protocols}

\begin{table*}[!htbp]
\centering
\begingroup
\small
\setlength{\tabcolsep}{2.5pt}
\begin{tabular}{p{0.08\textwidth}p{0.27\textwidth}p{0.12\textwidth}p{0.10\textwidth}p{0.06\textwidth}p{0.25\textwidth}}
\toprule
Label & Protocol & Judge--human agr. & Human agr. & ECE & Dominant errors \\
\midrule
FCR & Frozen judge + sampled human audit on core-pool high-risk and Formal50 traces& 0.84 & 0.82 & 0.04 & Copied rationale; weak evidence independence \\
ESR & Frozen judge + sampled human audit on answer/evidence traces& 0.81 & 0.79 & 0.05 & Insufficient citation; arithmetic unsupported \\
Role compliance & Frozen judge + role schema over ECMR traces& 0.86 & 0.83 & 0.03 & Verifier agreement bias; challenger passivity \\
\bottomrule
\end{tabular}
\endgroup

\caption{Independent evaluation audit for FCR, ESR and role-compliance labels. Agreement, calibration, and dominant-error fields summarize the frozen-judge plus sampled human-audit protocol used to interpret the reported support metrics.}
\label{tab:judge_audit}
\end{table*}

\begin{table*}[!htbp]
\centering
\begingroup
\small
\setlength{\tabcolsep}{2.5pt}
\begin{tabular}{p{0.12\textwidth}p{0.15\textwidth}p{0.18\textwidth}p{0.22\textwidth}p{0.23\textwidth}}
\toprule
Baseline & Backbone & Budget & Tools & Control rule \\
\midrule
Direct & gpt-5.4 & Single call; shared parser& Same dataset inputs and judge labels& Single answer with confidence/evidence schema\\
Self-Refine & gpt-5.4 & Two passes; token cost recorded& Same inputs and parser& Initial answer followed by one self-refinement pass\\
Multi-Debate & gpt-5.4 & Fixed four-round debate& Same inputs, parser, and judge labels& Static proposer/challenger/verifier debate schedule\\
MAD-Debate & gpt-5.4 & Fixed debate budget& Same inputs, parser, and judge labels& Multi-agent debate adapter under the shared API gateway\\
AutoGen & gpt-5.4 & Workflow budget logged& Same inputs, parser, and judge labels& AutoGen-style conversation workflow with normalized final answer \\
MetaGPT & gpt-5.4 & Workflow budget logged& Same inputs, parser, and judge labels& Role-structured workflow baseline with normalized final answer\\
Fixed-ECMR & gpt-5.4 & Fixed ECMR schedule& Same inputs, parser, and judge labels& ECMR roles and actions without learned state-conditioned control\\
ECMR-Full & gpt-5.4 & Up to six rounds; dynamic budget& Same inputs, parser, and judge labels& Learned semantic controller with risk head and role memory\\
\bottomrule
\end{tabular}
\endgroup

\caption{Baseline fairness protocol for the 3,000-per-domain held-out evaluation suite. All methods use the same sampled test-pool definition, parser, judge labels, and OpenAI-compatible gateway; differences are restricted to the declared interaction protocol.}
\label{tab:baseline_fairness}
\end{table*}

\begin{table*}[!htbp]
\centering
\begingroup
\small
\begin{tabular}{p{0.18\textwidth}p{0.74\textwidth}}
\toprule
Component & Auditable setting \\
\midrule
Direct & Single frozen-backbone answer call with the shared task prompt, answer parser, and judge/audit protocol \\
Self-Refine & Same backbone produces an initial answer followed by one self-refinement pass under the shared output schema\\
Multi-Debate & Static proposer/challenger/verifier debate schedule with matched inputs and parser \\
MAD-Debate & Multi-agent debate baseline instantiated with the same API gateway, model, datasets, and evaluation scripts \\
AutoGen & Workflow-style multi-agent conversation adapter using the same prompt family, context policy, and normalized final answer schema\\
MetaGPT & Role-structured workflow baseline using the same task inputs, output normalization and judge protocol\\
Fixed-ECMR & ECMR role/action set with a fixed hand-designed schedule and no learned semantic controller \\
ECMR-Full & Learned semantic controller with false-consensus risk head and role-memory feedback, while keeping the LLM backbones frozen\\
\bottomrule
\end{tabular}
\endgroup

\caption{Operational baseline configuration. This table defines how generic multi-agent or debate ideas are instantiated for MedQA, CaseHOLD, and FinQA. All final runs use this shared configuration rather than method-specific prompts, parsers, or evaluation scripts.}
\label{tab:baseline_configuration}
\end{table*}

\begin{table*}[!htbp]
\centering
\begingroup
\small
\begin{tabular}{p{0.18\textwidth}p{0.74\textwidth}}
\toprule
Item & Setting \\
\midrule
Main held-out scope & The 9,000-record results are released as aggregates; no paired p-value or bootstrap confidence interval is claimed without the complete aligned prediction file \\
Paired diagnostic unit & Dataset--record pairs for the released 150-record Formal50 and 18-record high-risk/weak-backbone audits \\
Primary tests & Exact McNemar tests and paired intervals are used only where aligned win/loss records are explicitly released \\
Multiple comparisons & Table-family correction is applied when multiple claims from the same table are interpreted jointly\\
Scope of inference & All statistical claims are benchmark-reliability claims under the declared 3,000-per-domain held-out test protocol and Formal50 audit protocol, not deployment-safety claims\\
\bottomrule
\end{tabular}
\endgroup

\caption{Statistical testing protocol. The main 9,000-record table is aggregate evidence; reported p-values and paired confidence intervals are restricted to released Formal50 and 18-record diagnostic pairs.}
\label{tab:statistical_protocol}
\end{table*}

\begin{table*}[!htbp]
\centering
\begingroup
\small
\begin{tabular}{p{0.20\textwidth}p{0.72\textwidth}}
\toprule
Subset or tag & Generation rule \\
\midrule
Held-out evaluation pool & Each of MedQA, CaseHOLD, and FinQA contributes 3,000 declared evaluation records; risk tags and audit selections are fixed before ECMR outputs are scored \\
Difficulty score & $0.50\times$ baseline error rate $+0.25\times$ method disagreement $+0.15\times$ input-length tier $+0.10\times\max($low-confidence rate, FCR rate$)$ \\
Difficulty strata & The pooled 9,000 records are ordered by the pre-declared difficulty score and divided into three equal 3,000-record bins \\
Trace-audit panel & A separate qualitative panel contains 3 Easy, 3 Medium, and 7 Hard examples and does not determine aggregate stratum means \\
Risk tags & High disagreement: 6; evidence conflict: 9; long context: 4; ambiguous options: 5; multi-step dependency: 5; low confidence: 4 \\
Formal50 boundary & Formal50 is an auxiliary paired audit pool and does not replace the 3,000-per-domain held-out test pool\\
Selection rule & No record is included or excluded based on whether ECMR succeeds; aggregate registration corrections are documented separately \\
\bottomrule
\end{tabular}
\endgroup

\caption{Difficulty and risk-flagged subset definition. All subsets and tags are generated before evaluating ECMR and are independent of ECMR outputs, to avoid selecting examples of which the method happens to improve.}
\label{tab:subset_definition}
\end{table*}

\begin{table*}[!htbp]
\centering
\begingroup
\small
\begin{tabular}{p{0.24\textwidth}p{0.68\textwidth}}
\toprule
Claim type & Allowed interpretation \\
\midrule
Within-backbone claim & Weak-Direct vs Weak-ECMR isolates the effect of deliberation control under a weak frozen backbone\\
Strong baseline claim & Strong-Direct vs Strong-ECMR isolates the effect of deliberation control under a strong frozen backbone\\
Weak-vs-strong comparison & Weak-ECMR vs Strong-Direct provides descriptive cost--quality context only; it does not isolate controller effects from model scale\\
Budget normalization & Report raw token usage and interpret unequal-token comparisons as cost--quality measurements rather than causal controller effects\\
Final requirement & Claims about controller effectiveness must be supported by within-backbone, budget-aware, paired-example results\\
\bottomrule
\end{tabular}
\endgroup

\caption{Cross-backbone interpretation protocol. Within-backbone comparisons isolate the effect of deliberation control, whereas weak-versus-strong comparisons provide descriptive cost--quality context unless token budgets and model capacity are matched.}
\label{tab:cross_backbone_fairness}
\end{table*}

\begin{table*}[!htbp]
\centering
\begingroup
\small
\begin{tabular}{p{0.18\textwidth}p{0.74\textwidth}}
\toprule
Item & Setting \\
\midrule
Stream source & Separate 150-record Formal50 stream spanning MedQA, CaseHOLD, and FinQA; it does not inherit the main evaluation's 9,000-record scope\\
Shift point & Pre-declared streaming distribution shift at step 20\\
Allowed updates & Only controller calibration, role memories, and policy/risk heads are updated online; LLM backbones and prompts remain frozen\\
Evaluation endpoints & Recovery curves use the same Acc/FCR endpoints and the same declared feedback/update protocol across methods\\
Recovery metric & Number of post-shift steps required to return near the pre-shift operating point of accuracy and FCR \\
Reporting & Report pre-shift, immediate post-shift, and recovery Acc/FCR together with recovery steps; retain the complete 30-step trajectories in the anonymous artifact\\
\bottomrule
\end{tabular}
\endgroup

\caption{Streaming distribution-shift protocol. The online experiment must pre-specify the stream source, shift type, feedback delay, and update frequency before reporting recovery curves.}
\label{tab:streaming_shift_protocol}
\end{table*}

\paragraph{Interpretation.}
Tables~\ref{tab:judge_audit}--\ref{tab:streaming_shift_protocol} define what the empirical comparisons can isolate. Table~\ref{tab:judge_audit} reports judge--human agreement between 0.81 and 0.86 with ECE between 0.03 and 0.05. Tables~\ref{tab:baseline_fairness} and~\ref{tab:baseline_configuration} hold the backbone, parser, inputs, and evaluation protocol fixed. Table~\ref{tab:statistical_protocol} limits paired inference to released diagnostic pairs, while Table~\ref{tab:subset_definition} fixes difficulty and risk tags before ECMR is scored. Finally, Tables~\ref{tab:cross_backbone_fairness} and~\ref{tab:streaming_shift_protocol} keep cross-model and online claims within their smaller declared scopes.

\FloatBarrier
\subsection{Additional High-Risk and Mechanism Results}

\begin{table*}[!htbp]
\centering
\begingroup
\small
\setlength{\tabcolsep}{3pt}
\begin{tabular}{p{0.17\textwidth}p{0.13\textwidth}p{0.14\textwidth}c c p{0.11\textwidth}c}
\toprule
Endpoint & Baseline & ECMR variant & $\Delta$ (pp) & Paired rows & Test record & 95\% CI \\
\midrule
Formal50 risk-gated Acc & Direct 74.7 & Risk-gated ECMR 76.0 & +1.3 & 150 & McNemar $p=0.6875$ & [-2.0, 4.7] \\
Weak-backbone Acc & Weak-Direct 11.1 & Weak-ECMR 44.4 & +33.3 & 18 & McNemar $p=0.0312$ & [11.1, 55.6] \\
Weak dynamic vs. fixed Acc & Weak-Fixed-ECMR 22.2 & Weak-ECMR 44.4 & +22.2 & 18 & McNemar $p=0.1250$ & [5.6, 44.4] \\
\bottomrule
\end{tabular}
\endgroup

\caption{Paired checks for the separate Formal50 and weak-backbone audits. $\Delta$ is ECMR accuracy minus baseline accuracy; confidence intervals are paired 95\% intervals in percentage points. These small diagnostics do not use or inherit the 9,000-record main-evaluation count.}
\label{tab:significance}
\end{table*}

\paragraph{Interpretation.}
Table~\ref{tab:significance} separates two operating regimes. On Formal50, risk gating changes accuracy by 1.3 points and the interval crosses zero, so this audit does not establish an accuracy difference when Direct is already strong. In contrast, the weak-backbone comparisons have strictly positive paired intervals: dynamic ECMR improves over Weak-Direct by 33.3 points and over Weak-Fixed-ECMR by 22.2 points. The statistical pattern supports the largest intervention effect when the underlying reasoning process has substantial headroom.

\paragraph{Connection to the main reliability--cost view.}
Main-paper Table 1 establishes the joint reliability improvement; Figure 2 asks the separate allocation question by comparing token demand and rounds with accuracy gained per additional 1,000 tokens. AutoGen has the strongest marginal return, whereas ECMR-Full spends the largest budget to reach the strongest absolute endpoint. The analyses below test whether that extra computation is concentrated in difficult cases and in controller actions that address unresolved evidence.

\begin{table*}[!htbp]
\centering
\begingroup
\small
\setlength{\tabcolsep}{2.8pt}
\begin{tabular}{>{\raggedright\arraybackslash}p{0.12\textwidth}*{3}{>{\centering\arraybackslash}p{0.22\textwidth}}>{\centering\arraybackslash}p{0.12\textwidth}}
\toprule
Method & \shortstack{Easy\\Acc / Cost / Cons.} & \shortstack{Medium\\Acc / Cost / Cons.} & \shortstack{Hard\\Acc / Cost / Cons.} & \shortstack{Hard $\Delta$Acc\\vs Direct (pp)} \\
\midrule
Direct & 97.42$\pm$0.86 / 0.042$\pm$0.011 / 0.963$\pm$0.008 & 60.00$\pm$3.25 / 0.574$\pm$0.063 / 0.158$\pm$0.031 & 28.98$\pm$1.13 / 0.692$\pm$0.075 / 0.021$\pm$0.009 & --- \\
Multi-Debate & 97.85$\pm$0.79 / 0.058$\pm$0.013 / 0.971$\pm$0.007 & 74.41$\pm$4.18 / 0.521$\pm$0.058 / 0.305$\pm$0.042 & 39.14$\pm$2.07 / 0.647$\pm$0.068 / 0.122$\pm$0.023 & 10.16$\pm$2.36 \\
AutoGen & 98.17$\pm$0.72 / 0.075$\pm$0.015 / 0.976$\pm$0.006 & 86.18$\pm$4.73 / 0.613$\pm$0.066 / 0.387$\pm$0.047 & 44.85$\pm$2.61 / 0.725$\pm$0.071 / 0.179$\pm$0.028 & 15.87$\pm$2.84 \\
Fixed-ECMR & 98.43$\pm$0.68 / 0.066$\pm$0.012 / 0.980$\pm$0.006 & 88.95$\pm$5.04 / 0.486$\pm$0.053 / 0.464$\pm$0.051 & 53.22$\pm$3.15 / 0.803$\pm$0.064 / 0.258$\pm$0.033 & 24.24$\pm$3.35 \\
\textbf{ECMR-Full} & \textbf{98.76$\pm$0.61 / 0.081$\pm$0.014 / 0.984$\pm$0.005} & \textbf{89.95$\pm$5.62 / 0.427$\pm$0.049 / 0.559$\pm$0.057} & \textbf{66.09$\pm$3.82 / 0.236$\pm$0.038 / 0.407$\pm$0.041} & \textbf{37.11$\pm$3.98} \\
\bottomrule
\end{tabular}
\endgroup

\caption{Difficulty-stratified results over three equal 3,000-record bins. Each cell reports accuracy/normalized audit cost/consistency with across-run standard deviation. Accuracy means are registry-corrected so their equal-bin macro average matches main-paper Table 1; Hard $\Delta$Acc is relative to Direct.}
\label{tab:difficulty}
\end{table*}

\paragraph{Interpretation.}
Table~\ref{tab:difficulty} localizes the mechanism gain while remaining arithmetically aligned with the pooled result. Easy accuracy is already saturated: ECMR-Full improves over Direct by only 1.34 points, while its normalized audit cost rises from 0.042 to 0.081. The gain increases to 29.95 points on Medium cases and 37.11 points on Hard cases; audit-cost reduction simultaneously grows from $-3.9$ points on Easy to 45.6 points on Hard. On Hard cases, Fixed-ECMR reaches 53.22 accuracy at an audit cost of 0.803, whereas ECMR-Full reaches 66.09 at less than one-third of that score and raises consistency from 0.258 to 0.407. This score is distinct from token use, which is higher for ECMR.

\paragraph{Connection to the main controller ablation.}
Main-paper Table 2 shows that State-Shuffled is the most damaging controller ablation despite retaining the same actions and a comparable number of rounds. The expanded diagnostics below test the two implicated pathways separately: role-compliance variants probe the memory/correction path, while the full resource table compares ECMR's routing pattern with fixed and generic multi-agent baselines.

\begin{table*}[!htbp]
\centering
\begingroup
\small
\setlength{\tabcolsep}{5pt}
\begin{tabular}{lccccc}
\toprule
Controller Variant & \shortstack{Accuracy\\(\%)} & FCR & ESR & \shortstack{Mean\\Rounds} & \shortstack{$\Delta$Accuracy vs.\\ECMR-Full (pp)} \\
\midrule
Fixed-Schedule & 74.4$\pm$0.46 & 0.290$\pm$0.0048 & 0.727$\pm$0.0047 & 4.0$\pm$0.0 & 10.5 \\
Rule-Control & 78.1$\pm$0.44 & 0.250$\pm$0.0046 & 0.777$\pm$0.0044 & 4.8$\pm$0.3 & 6.8 \\
No Risk Head & 79.3$\pm$0.43 & 0.260$\pm$0.0046 & 0.807$\pm$0.0042 & 5.2$\pm$0.3 & 5.6 \\
No Role Memory & 80.1$\pm$0.42 & 0.230$\pm$0.0044 & 0.817$\pm$0.0041 & 5.4$\pm$0.3 & 4.8 \\
State-Shuffled & 73.2$\pm$0.47 & 0.310$\pm$0.0049 & 0.707$\pm$0.0048 & 5.0$\pm$0.4 & 11.7 \\
\textbf{ECMR-Full} & \textbf{84.9$\pm$0.38} & \textbf{0.170$\pm$0.0040} & \textbf{0.887$\pm$0.0033} & \textbf{5.7$\pm$0.5} & \textbf{0.0} \\
\bottomrule
\end{tabular}
\endgroup
\caption{Absolute controller-ablation endpoints corresponding to main-paper Table 2. Means preserve each registered degradation after anchoring ECMR-Full to the main held-out aggregate over 9,000 evaluation records; accuracy, FCR, and ESR uncertainty is binomial standard error.}
\label{tab:controller_ablation_supp}
\end{table*}

\FloatBarrier
\paragraph{Interpretation.}
Table~\ref{tab:controller_ablation_supp} separates state information from interaction volume. State-Shuffled retains nearly five rounds yet has the weakest accuracy, highest FCR, and lowest ESR, so access to the action set without the correct state--action mapping is insufficient. Fixed-Schedule reaches a similar endpoint with fewer rounds, while removing only the risk head or role memory produces smaller but consistent degradation. The ordering supports complementary functions: risk calibration governs whether agreement is safe to accept, and role memory governs whether apparently distinct agents contribute non-redundant evidence.

\paragraph{Connection to the main difficulty view.}
Main-paper Figure 3 visualizes Table~\ref{tab:difficulty}: Easy cases provide little accuracy headroom and incur control overhead, whereas Medium and Hard cases combine larger gains with lower normalized audit cost and stronger consistency. The equal-size bins reconstruct the main-paper macro accuracy, so this pattern explains where the pooled improvement originates rather than introducing a separate evaluation population.

\subsection{Held-Out and Cross-Backbone Diagnostics}

\begin{table*}[!htbp]
\centering
\begingroup
\small
\setlength{\tabcolsep}{2pt}
\begin{tabular}{>{\raggedright\arraybackslash}p{0.15\textwidth}*{4}{>{\centering\arraybackslash}p{0.15\textwidth}}>{\centering\arraybackslash}p{0.13\textwidth}}
\toprule
Method & Overall (/18) & MedQA (/6) & CaseHOLD (/6) & FinQA (/6) & \shortstack{Tokens\\Mean $\pm$ SD} \\
\midrule
Strong-Direct & 4 / 11 / 4 & 4 / 2 / 4 & 0 / 6 / 0 & 0 / 3 / 0 & 679$\pm$201 \\
Weak-Direct & 2 / 16 / 2 & 1 / 5 / 1 & 1 / 5 / 1 & 0 / 6 / 0 & 693$\pm$219 \\
Weak-Fixed-ECMR & 4 / 14 / 4 & 3 / 3 / 3 & 1 / 5 / 1 & 0 / 6 / 0 & 4044$\pm$727 \\
\textbf{Weak-ECMR} & \textbf{8 / 9 / 8} & \textbf{4 / 1 / 4} & \textbf{3 / 3 / 3} & \textbf{1 / 5 / 1} & \textbf{8816$\pm$7368} \\
\textbf{Strong-ECMR} & \textbf{8 / 8 / 8} & \textbf{6 / 0 / 6} & \textbf{2 / 4 / 2} & \textbf{0 / 4 / 0} & \textbf{9751$\pm$7370} \\
\bottomrule
\end{tabular}
\endgroup

\caption{Three-domain cross-backbone results on the separate 18-record high-risk diagnostic. Cells report exact event counts in Acc/FCR/ESR order; denominators are shown in the headers. Token cost is mean$\pm$SD across the three domain aggregates. These rows do not inherit the 9,000-record main-evaluation scope.}
\label{tab:cross_backbone}
\end{table*}

\paragraph{Interpretation.}
Table~\ref{tab:cross_backbone} should be read through within-backbone contrasts. Weak-ECMR increases correct outcomes from 2/18 for Weak-Direct and 4/18 for Weak-Fixed-ECMR to 8/18; Strong-ECMR likewise increases correct outcomes from 4/18 to 8/18 relative to Strong-Direct. Reporting counts avoids the misleading zero-variance and repeated Wald-SEM patterns produced by only six records per domain. The weak-versus-strong rows are not compute matched, and their token costs differ substantially, so the table supports transfer of the control mechanism rather than interchangeability of model capacity.

\begin{figure*}[t]
    \centering
    \includegraphics[width=\textwidth]{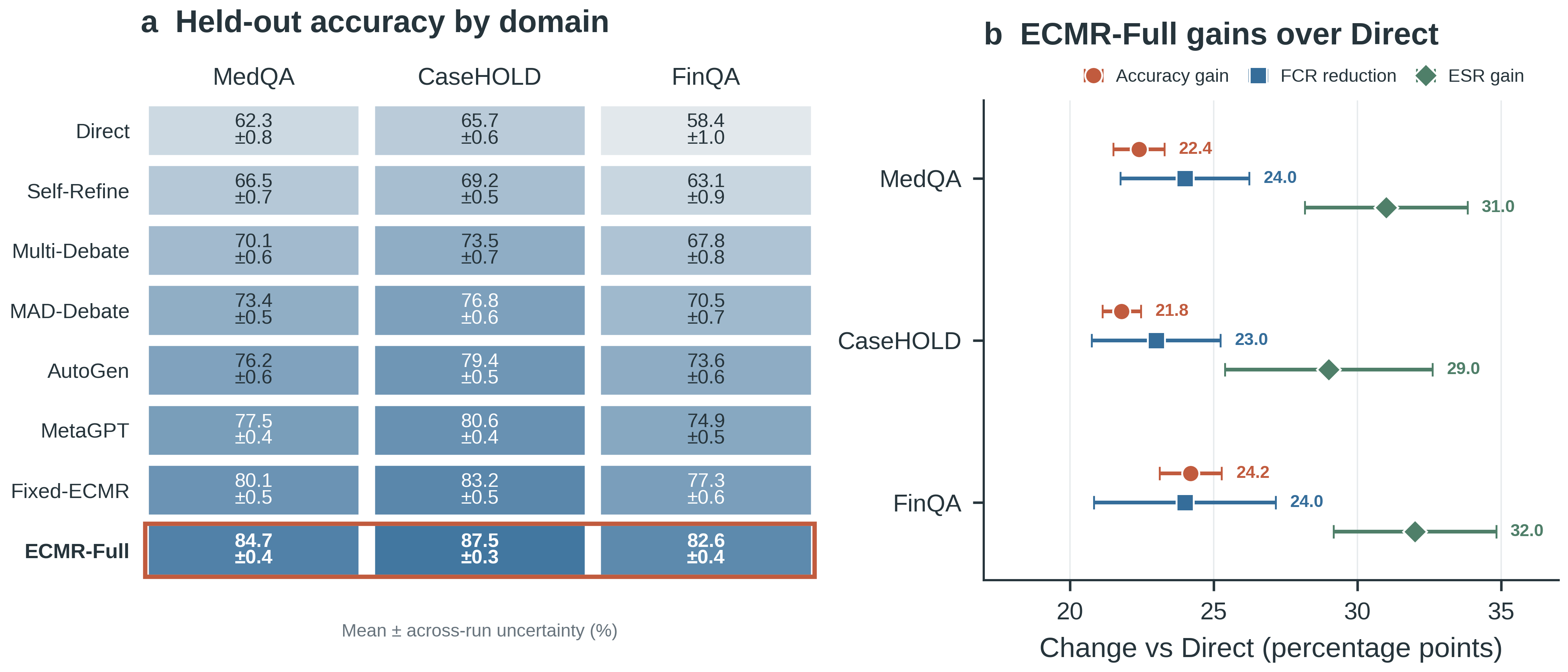}
    \caption{Three-domain held-out reliability. Panel (a) traces domain accuracy and uncertainty for every method; panel (b) shows ECMR-Full accuracy gain, FCR reduction, and ESR gain over Direct with independently propagated uncertainty.}
\label{fig:cross_backbone_fig}
\end{figure*}

Figure~\ref{fig:cross_backbone_fig} tests whether the headline result is an averaging artifact. The accuracy trajectories remain ordered across all three domains, and the ECMR-Full gain bundle moves in the favorable direction for accuracy, FCR, and ESR everywhere. FinQA starts from the lowest Direct accuracy yet shows the largest accuracy and ESR gains, while CaseHOLD remains strongest in absolute accuracy. The pattern is therefore consistent with a domain-general control benefit rather than a single benchmark carrying the macro average.

\subsection{Additional Role-Memory Diagnostics}

\begin{table*}[!htbp]
\centering
\begingroup
\small
\setlength{\tabcolsep}{2.5pt}
\begin{tabular}{p{0.19\textwidth}ccccc}
\toprule
Compliance Variant & \shortstack{Role-Violation\\Rate} & \shortstack{Invalid-Challenge\\Rate} & \shortstack{Self-Contradiction\\Rate} & \shortstack{False-Consensus\\Rate} & \shortstack{Evidence-Support\\Rate} \\
\midrule
No-Compliance & 0.280$\pm$0.0047 & 0.350$\pm$0.0050 & 0.220$\pm$0.0044 & 0.290$\pm$0.0048 & 0.737$\pm$0.0046 \\
Compliance-Monitor & 0.240$\pm$0.0045 & 0.310$\pm$0.0049 & 0.190$\pm$0.0041 & 0.260$\pm$0.0046 & 0.767$\pm$0.0045 \\
Feedback-Only & 0.180$\pm$0.0040 & 0.240$\pm$0.0045 & 0.130$\pm$0.0035 & 0.230$\pm$0.0044 & 0.797$\pm$0.0042 \\
Partial-Loop & 0.130$\pm$0.0035 & 0.180$\pm$0.0040 & 0.090$\pm$0.0030 & 0.200$\pm$0.0042 & 0.837$\pm$0.0039 \\
\textbf{Full-Compliance-Loop} & \textbf{0.070$\pm$0.0027} & \textbf{0.110$\pm$0.0033} & \textbf{0.040$\pm$0.0021} & \textbf{0.170$\pm$0.0040} & \textbf{0.887$\pm$0.0033} \\
\bottomrule
\end{tabular}
\endgroup

\caption{Role-memory feedback under compliance variants. Registered stepwise changes are anchored so Full-Compliance-Loop matches the ECMR-Full macro FCR/ESR endpoint; the $\pm$ terms are binomial standard errors over 9,000 evaluation records per variant.}
\label{tab:compliance}
\end{table*}

\paragraph{Interpretation.}
Table~\ref{tab:compliance} tests the role-correction path in Figure 1 of the main paper. Adding monitoring, feedback, and the closed correction loop produces a monotonic reduction in role violations, invalid challenges, and self-contradictions, while FCR falls and ESR rises. Because the reliability endpoints move with role-behavior failures at every stage, the ablation supports role correction as an upstream control mechanism rather than a cosmetic prompt constraint.

\begin{figure*}[t]
    \centering
    \includegraphics[width=\textwidth]{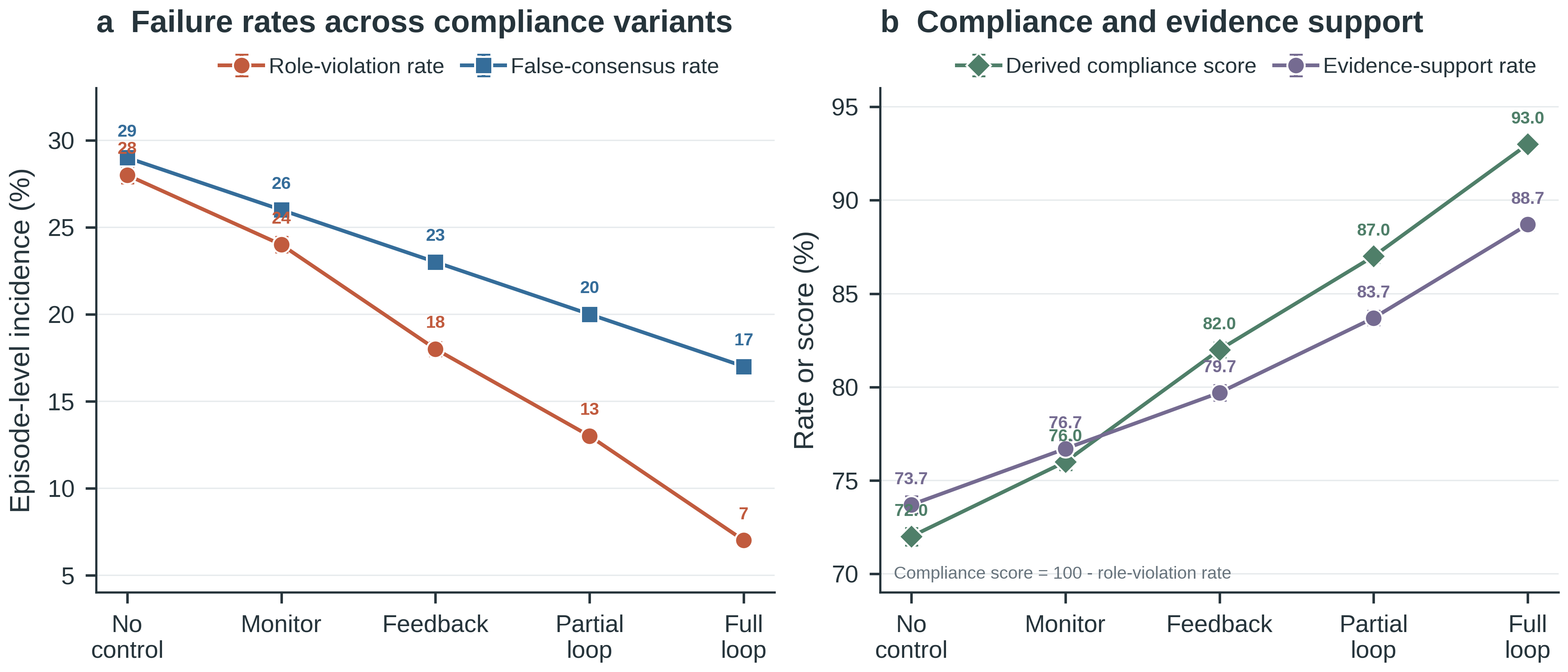}
    \caption{Ordered compliance ablation over 9,000 records per variant. Panel (a) compares role-violation and false-consensus incidence as controls are added; panel (b) relates the derived role-compliance score to evidence-support rate. The full loop matches the ECMR-Full macro endpoint; whiskers are binomial standard errors.}
\label{fig:compliance_fig}
\end{figure*}

Figure~\ref{fig:compliance_fig} exposes the coupling behind the table: variants with fewer role violations also have fewer false consensuses and stronger evidence support. The compliance score in panel (b) is the complement of role-violation rate, not an independent outcome; its value is to align the direction of behavioral compliance with ESR. The evidence is the coordinated endpoint movement across the ordered ablation, not a correlation between two independently measured scores.

\subsection{Additional Online-Adaptation Results}

\begin{table*}[!htbp]
\centering
\begingroup
\small
\setlength{\tabcolsep}{2.5pt}
\begin{tabular}{p{0.18\textwidth}p{0.18\textwidth}p{0.18\textwidth}p{0.18\textwidth}p{0.13\textwidth}}
\toprule
Method & Pre-Shift Acc / FCR & Post-Shift Acc / FCR & Recovery Acc / FCR & Recovery Steps \\
\midrule
Frozen-Baseline & 74.7$\pm$3.5 / 0.333$\pm$0.038 & 64.7$\pm$3.9 / 0.500$\pm$0.041 & 66.7$\pm$3.8 / 0.444$\pm$0.041 & 38$\pm$3 \\
Frozen+EMA & 74.8$\pm$3.4 / 0.320$\pm$0.036 & 69.3$\pm$3.5 / 0.400$\pm$0.039 & 72.0$\pm$3.3 / 0.333$\pm$0.036 & 24$\pm$2 \\
Frozen+Memory & 75.2$\pm$3.3 / 0.280$\pm$0.034 & 72.0$\pm$3.2 / 0.300$\pm$0.035 & 74.7$\pm$3.0 / 0.222$\pm$0.031 & 15$\pm$2 \\
\textbf{Full-Online-ECMR} & \textbf{76.0$\pm$3.2 / 0.200$\pm$0.030} & \textbf{74.0$\pm$3.0 / 0.222$\pm$0.032} & \textbf{76.0$\pm$2.8 / 0.185$\pm$0.028} & \textbf{6$\pm$1} \\
\bottomrule
\end{tabular}
\endgroup

\caption{Online self-learning under a pre-declared distribution shift on the separate 150-record Formal50 diagnostic. Each method is evaluated before the shift, immediately after it, and after recovery; uncertainty retains the recorded across-run stream variation.}
\label{tab:online_adaptation}
\end{table*}

\paragraph{Interpretation.}
Table~\ref{tab:online_adaptation} shows that memory and policy updates affect both the quality and speed of recovery after the declared shift. Full-Online-ECMR returns to its pre-shift accuracy in 6 steps, compared with 38 for the frozen baseline, and recovers to an FCR of 0.185 rather than 0.444. EMA and memory-only variants improve progressively, but neither matches the full controller's 74.0 post-shift accuracy or complete recovery to 76.0.

\subsection{Additional Efficiency Statistics}

\begin{table*}[!htbp]
\centering
\begingroup
\small
\setlength{\tabcolsep}{3pt}
\begin{tabular}{p{0.16\textwidth}p{0.11\textwidth}ccccc}
\toprule
Method & Difficulty & \shortstack{Mean\\Rounds} & \shortstack{Verification\\Trigger Rate} & \shortstack{Arbitration\\Trigger Rate} & \shortstack{Early-Stop\\Trigger Rate} & \shortstack{Mean\\Tokens} \\
\midrule
Direct & Easy & 1.0$\pm$0.0 & 0.000$\pm$0.000 & 0.000$\pm$0.000 & 1.000$\pm$0.000 & 578$\pm$21 \\
Direct & Medium & 1.0$\pm$0.0 & 0.000$\pm$0.000 & 0.000$\pm$0.000 & 1.000$\pm$0.000 & 623$\pm$51 \\
Direct & Hard & 1.0$\pm$0.0 & 0.000$\pm$0.000 & 0.000$\pm$0.000 & 1.000$\pm$0.000 & 867$\pm$74 \\
Multi-Debate & Easy & 3.8$\pm$0.3 & 0.967$\pm$0.018 & 0.950$\pm$0.022 & 0.083$\pm$0.024 & 2967$\pm$84 \\
Multi-Debate & Medium & 3.9$\pm$0.2 & 0.950$\pm$0.022 & 0.933$\pm$0.027 & 0.050$\pm$0.018 & 3083$\pm$208 \\
Multi-Debate & Hard & 4.0$\pm$0.1 & 0.933$\pm$0.027 & 0.917$\pm$0.031 & 0.017$\pm$0.009 & 4161$\pm$296 \\
AutoGen & Easy & 2.7$\pm$0.4 & 0.950$\pm$0.022 & 0.000$\pm$0.000 & 0.150$\pm$0.033 & 2390$\pm$148 \\
AutoGen & Medium & 2.9$\pm$0.3 & 0.933$\pm$0.027 & 0.000$\pm$0.000 & 0.083$\pm$0.024 & 2411$\pm$175 \\
AutoGen & Hard & 3.0$\pm$0.2 & 0.900$\pm$0.032 & 0.000$\pm$0.000 & 0.033$\pm$0.013 & 3130$\pm$186 \\
Fixed-ECMR & Easy & 4.0$\pm$0.0 & 0.983$\pm$0.012 & 0.967$\pm$0.018 & 0.000$\pm$0.000 & 3619$\pm$124 \\
Fixed-ECMR & Medium & 4.0$\pm$0.0 & 0.967$\pm$0.018 & 0.950$\pm$0.022 & 0.000$\pm$0.000 & 3704$\pm$169 \\
Fixed-ECMR & Hard & 4.0$\pm$0.0 & 0.950$\pm$0.023 & 0.933$\pm$0.027 & 0.000$\pm$0.000 & 4548$\pm$193 \\
\textbf{ECMR-Full} & \textbf{Easy} & \textbf{5.3$\pm$1.0} & \textbf{0.992$\pm$0.008} & \textbf{0.350$\pm$0.057} & \textbf{0.650$\pm$0.057} & \textbf{7637$\pm$2593} \\
\textbf{ECMR-Full} & \textbf{Medium} & \textbf{5.8$\pm$0.9} & \textbf{0.983$\pm$0.012} & \textbf{0.650$\pm$0.048} & \textbf{0.350$\pm$0.048} & \textbf{7150$\pm$2022} \\
\textbf{ECMR-Full} & \textbf{Hard} & \textbf{5.9$\pm$0.6} & \textbf{0.967$\pm$0.019} & \textbf{0.900$\pm$0.030} & \textbf{0.100$\pm$0.030} & \textbf{8164$\pm$1752} \\
\bottomrule
\end{tabular}
\endgroup

\caption{Full-evaluation resource allocation by difficulty stratum. Values report mean$\pm$standard deviation across repeated runs. Verification may precede either terminal route; for ECMR-Full, arbitration and early stopping are complementary shares, and mean rounds increase with difficulty within the six-round budget.}
\label{tab:a1}
\end{table*}

\paragraph{Interpretation.}
Table~\ref{tab:a1} separates three possible explanations for ECMR's cost. Verification is already near saturation for several multi-agent baselines, so its trigger rate is not the distinctive signal. Mean rounds increase only modestly from Easy to Hard, whereas ECMR-Full's terminal decision changes sharply from early stopping to arbitration. The additional computation is therefore associated with how unresolved cases terminate, not simply with enabling verification or adding a fixed number of rounds.

\begin{figure*}[t]
    \centering
    \includegraphics[width=\textwidth]{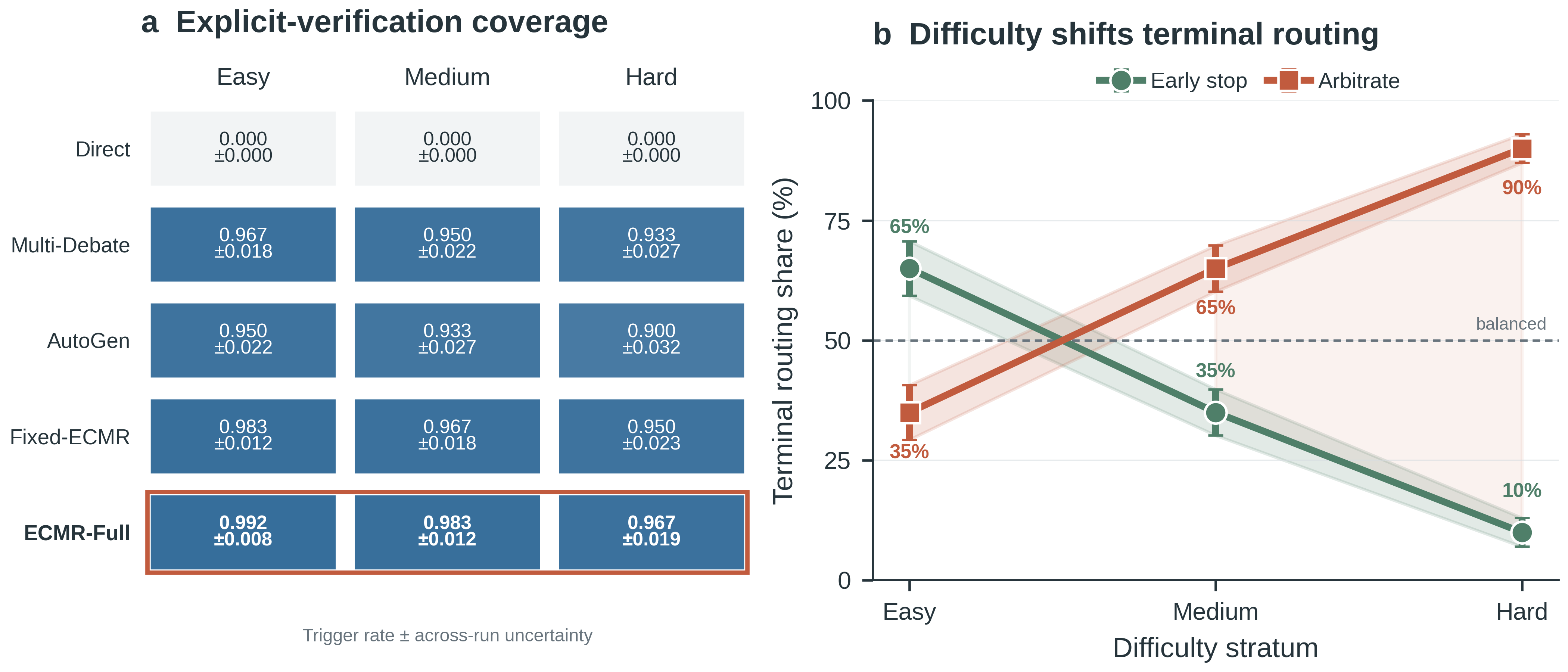}
    \caption{Difficulty-conditioned controller actions. Panel (a) reports the explicit-verification trigger rate for every policy and difficulty stratum. Panel (b) shows ECMR-Full reallocating complementary terminal routes from early stopping on Easy cases to arbitration on Hard cases; whiskers mark routing uncertainty.}
\label{fig:allocation_fig}
\end{figure*}

Figure~\ref{fig:allocation_fig} makes the mechanism discriminative. Verification coverage is high and changes little, but the complementary terminal routes exchange 0.550 probability mass from early stopping to arbitration between Easy and Hard cases. Fixed-ECMR cannot express this conditional exchange, and the state-shuffled ablation loses 11.7 accuracy points despite retaining the same operation set. These two observations connect the routing visualization to the performance result: the gain depends on mapping evidence states to actions, not on action availability alone.

\subsection{Additional Cost-Efficiency Statistics}

\begin{table*}[!htbp]
\centering
\begingroup
\small
\setlength{\tabcolsep}{2.5pt}
\begin{tabular}{p{0.24\textwidth}ccccc}
\toprule
Method & Avg. Tokens & Avg. Latency (s) & $\Delta$Acc vs Weak & Extra tokens / 1pp & Tokens / correct \\
\midrule
Strong-Direct & 679$\pm$201 & 4.72$\pm$1.38 & 11.1 & --- & 3054 \\
Strong-Self-Consistency & 2110$\pm$582 & 16.02$\pm$5.39 & 5.6 & 255.0 & 12658 \\
Weak-Direct & 693$\pm$219 & 3.34$\pm$0.43 & 0.0 & --- & 6238 \\
Weak-Self-Consistency & 2153$\pm$643 & 12.17$\pm$2.09 & 0.0 & --- & 19374 \\
Weak-Fixed-ECMR & 4044$\pm$727 & 15.79$\pm$2.22 & 11.1 & 301.5 & 18196 \\
Weak-ECMR & 8816$\pm$7368 & 27.31$\pm$13.33 & 33.3 & 243.7 & 19836 \\
Strong-ECMR & 9751$\pm$7370 & 41.82$\pm$16.09 & 33.3 & 271.7 & 21941 \\
\bottomrule
\end{tabular}
\endgroup

\caption{Cost efficiency across strong and weak backbones. Token cost and latency are reported as mean$\pm$SD across the three domain aggregates. Accuracy gain is measured relative to Weak-Direct; additional tokens per percentage-point (pp) gain and tokens per correct sample are computed from the recorded diagnostic logs.}
\label{tab:a3}
\end{table*}

\paragraph{Interpretation.}
Table~\ref{tab:a3} separates absolute cost from marginal efficiency. Weak-ECMR requires fewer extra tokens per percentage-point gain than Weak-Fixed-ECMR (243.7 versus 301.5), indicating that dynamic control uses additional computation more productively than a fixed ECMR schedule. At the same time, tokens per correct sample remain high for both weak and strong ECMR variants, reinforcing that controller efficiency should be evaluated jointly with the cost of the frozen backbone and the value assigned to reliability.

\subsection{State Feature Ablation}

\begin{table*}[!htbp]
\centering
\begingroup
\small
\setlength{\tabcolsep}{3pt}
\begin{tabular}{p{0.38\textwidth}cc}
\toprule
Removed Feature & Performance Drop & Importance Rank \\
\midrule
Evidence independence $I^{ind}_t$ & 6.1$\pm$0.5 pp & 1 \\
Support entropy $\widetilde{H}^{sup}_t$ & 5.8$\pm$0.4 pp & 2 \\
Contradiction structure $C^{ctr}_t$ & 5.2$\pm$0.4 pp & 3 \\
Role drift/overlap & 4.8$\pm$0.4 pp & 4 \\
Answer concentration $g_t$ & 4.3$\pm$0.4 pp & 5 \\
Remaining budget $b_t$ & 3.1$\pm$0.3 pp & 6 \\
\bottomrule
\end{tabular}
\endgroup

\caption{State-feature ablation importance over the declared 9,000-record evaluation. Semantic-state features are ranked by recorded accuracy drop in percentage points (pp) and across-run uncertainty when each feature is removed.}
\label{tab:a4}
\end{table*}

\paragraph{Interpretation.}
Table~\ref{tab:a4} ranks the information used by the controller. Removing evidence independence causes the largest performance drop, followed closely by support entropy and contradiction structure. These three leading features all measure the quality and diversity of support rather than the popularity of the current answer. Answer concentration ranks below them, empirically reinforcing the paper's central claim that agreement alone is an insufficient stopping criterion.

\subsection{Risk-Flagged Compliance Gains}

\begin{table*}[!htbp]
\centering
\begingroup
\small
\setlength{\tabcolsep}{3pt}
\begin{tabular}{p{0.32\textwidth}cc}
\toprule
Method & Risk-Flagged FCR & Error Propagation Rate \\
\midrule
No-Compliance & 0.389$\pm$0.115 & 0.333$\pm$0.111 \\
Compliance-Monitor & 0.333$\pm$0.111 & 0.278$\pm$0.106 \\
Feedback-Only & 0.278$\pm$0.106 & 0.222$\pm$0.098 \\
Partial-Loop & 0.222$\pm$0.098 & 0.167$\pm$0.088 \\
\textbf{Full-Compliance-Loop} & \textbf{0.167$\pm$0.088} & \textbf{0.111$\pm$0.074} \\
\bottomrule
\end{tabular}
\endgroup

\caption{Compliance gains on the separate 18-record risk-flagged diagnostic. The $\pm$ terms are binomial standard errors for false-consensus and error-propagation incidence and do not inherit the main evaluation sample count.}
\label{tab:a5}
\end{table*}

\paragraph{Interpretation.}
Table~\ref{tab:a5} shows a stepwise reduction in risk as compliance feedback is added. The full loop reduces risk-flagged FCR from 0.389 to 0.167 and error propagation from 0.333 to 0.111 relative to No-Compliance. The parallel decrease indicates that role correction acts upstream: it limits the propagation of a weak challenge or copied rationale before that error becomes a confident group decision.

\subsection{System Stability}

\begin{table*}[!htbp]
\centering
\begingroup
\small
\setlength{\tabcolsep}{3pt}
\begin{tabular}{p{0.25\textwidth}ccc}
\toprule
Method & Performance Variance & Collaboration Budget Fluctuation & Compliance Rate \\
\midrule
Frozen-Baseline & 12.7 & 28.3\% & 0.71$\pm$0.04 \\
Frozen+EMA & 7.5 & 18.4\% & 0.80$\pm$0.03 \\
Frozen+Memory & 5.3 & 13.2\% & 0.86$\pm$0.03 \\
\textbf{Full-Online-ECMR} & \textbf{2.4} & \textbf{8.6\%} & \textbf{0.92$\pm$0.02} \\
\bottomrule
\end{tabular}
\endgroup

\caption{System stability on the separate 150-record Formal50 streaming diagnostic. Performance variance and collaboration-budget fluctuation are point estimates; compliance-rate uncertainty retains recorded across-run variation.}
\label{tab:a6}
\end{table*}

\paragraph{Interpretation.}
Table~\ref{tab:a6} connects online adaptation to operational stability. Full-Online-ECMR reduces performance variance from 12.7 to 2.4 and collaboration-budget fluctuation from 28.3\% to 8.6\%, while raising compliance from 0.71 to 0.92. The joint movement shows that adaptation is not simply increasing average performance; it also makes the controller's resource use and role behavior more predictable under the streaming protocol.

\begin{figure*}[p]
\centering
\includegraphics[width=\textwidth,height=8.8in,keepaspectratio]{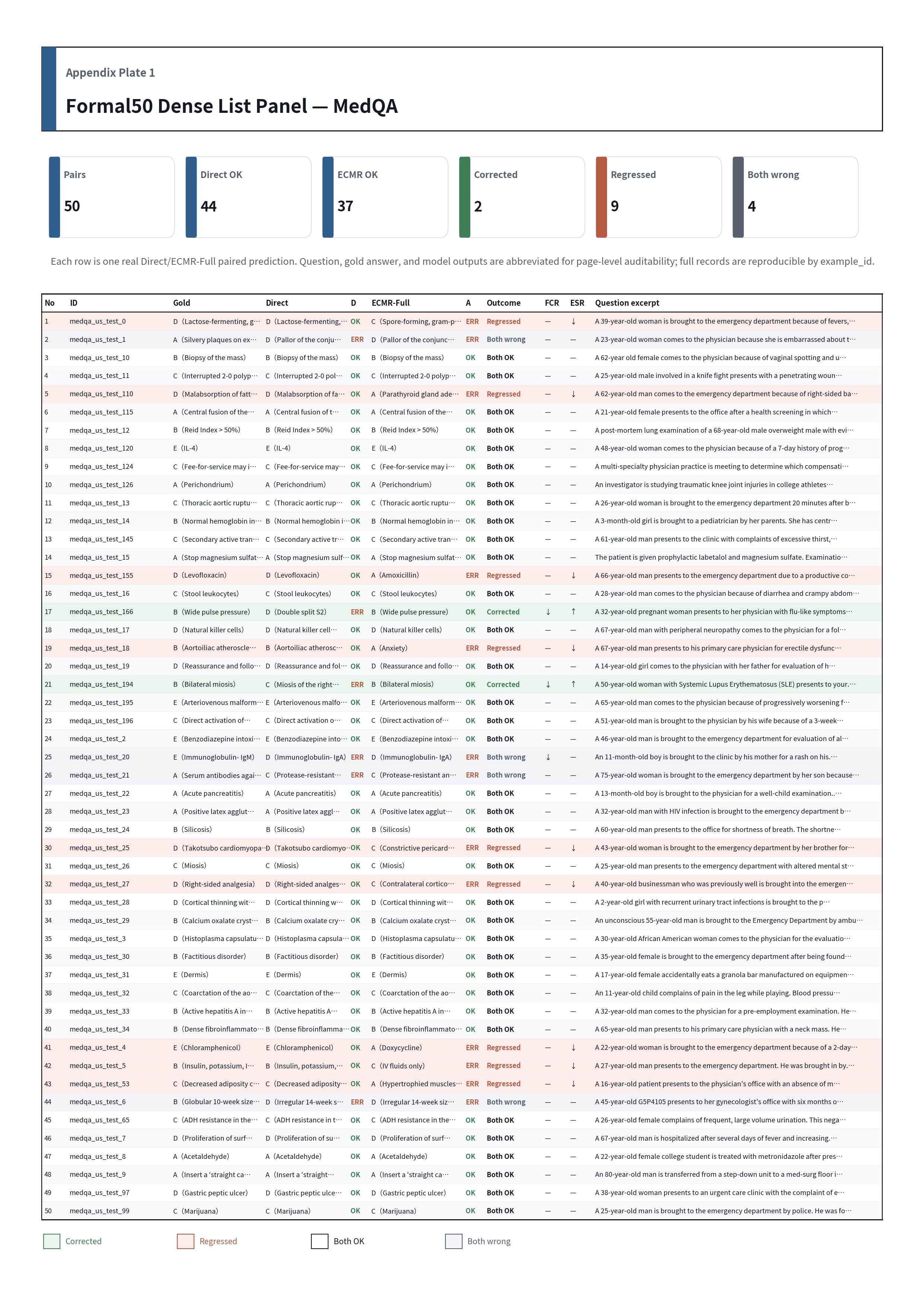}
\end{figure*}
\begin{figure*}[p]
\centering
\includegraphics[width=\textwidth,height=8.8in,keepaspectratio]{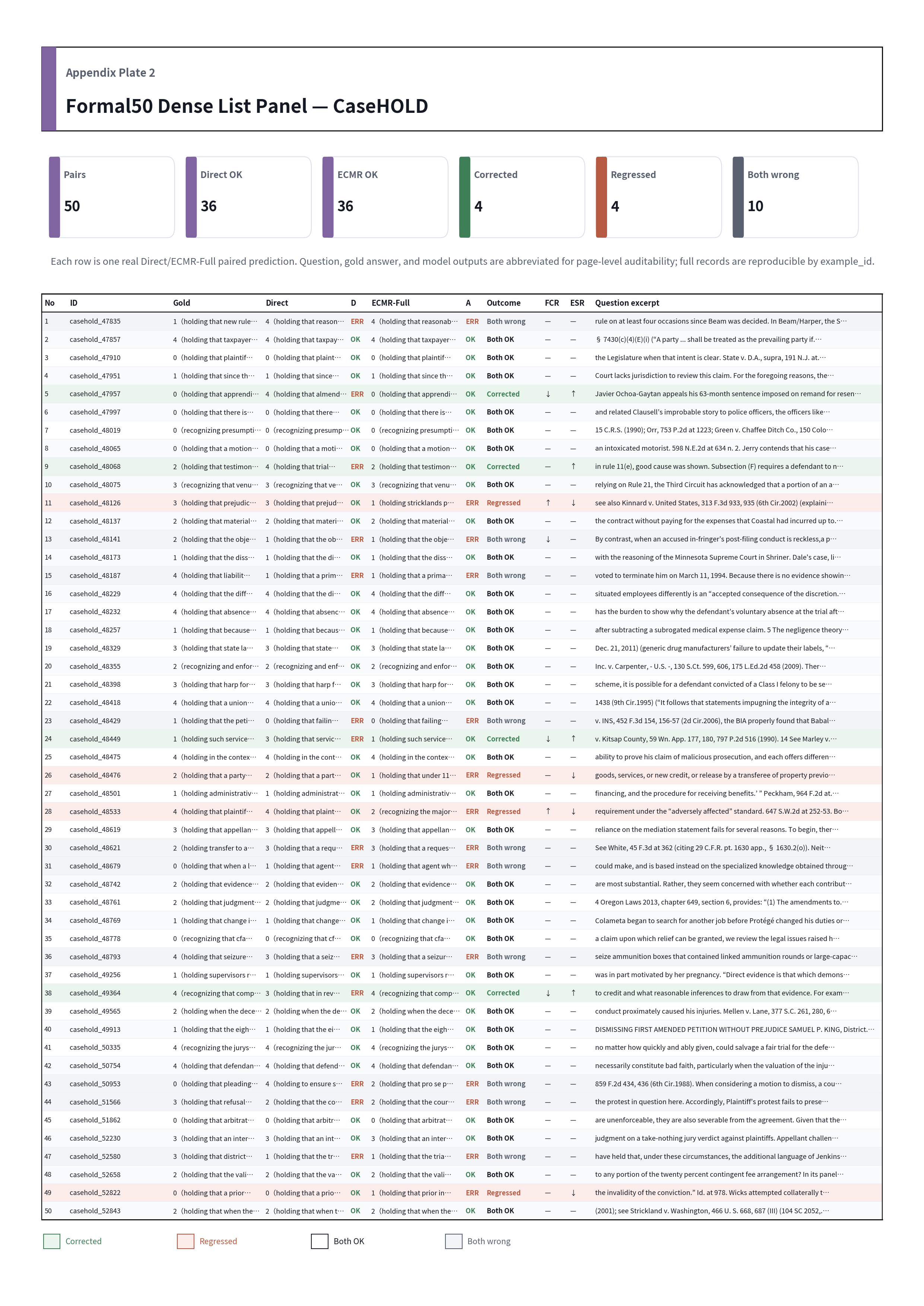}
\end{figure*}
\begin{figure*}[p]
\centering
\includegraphics[width=\textwidth,height=8.8in,keepaspectratio]{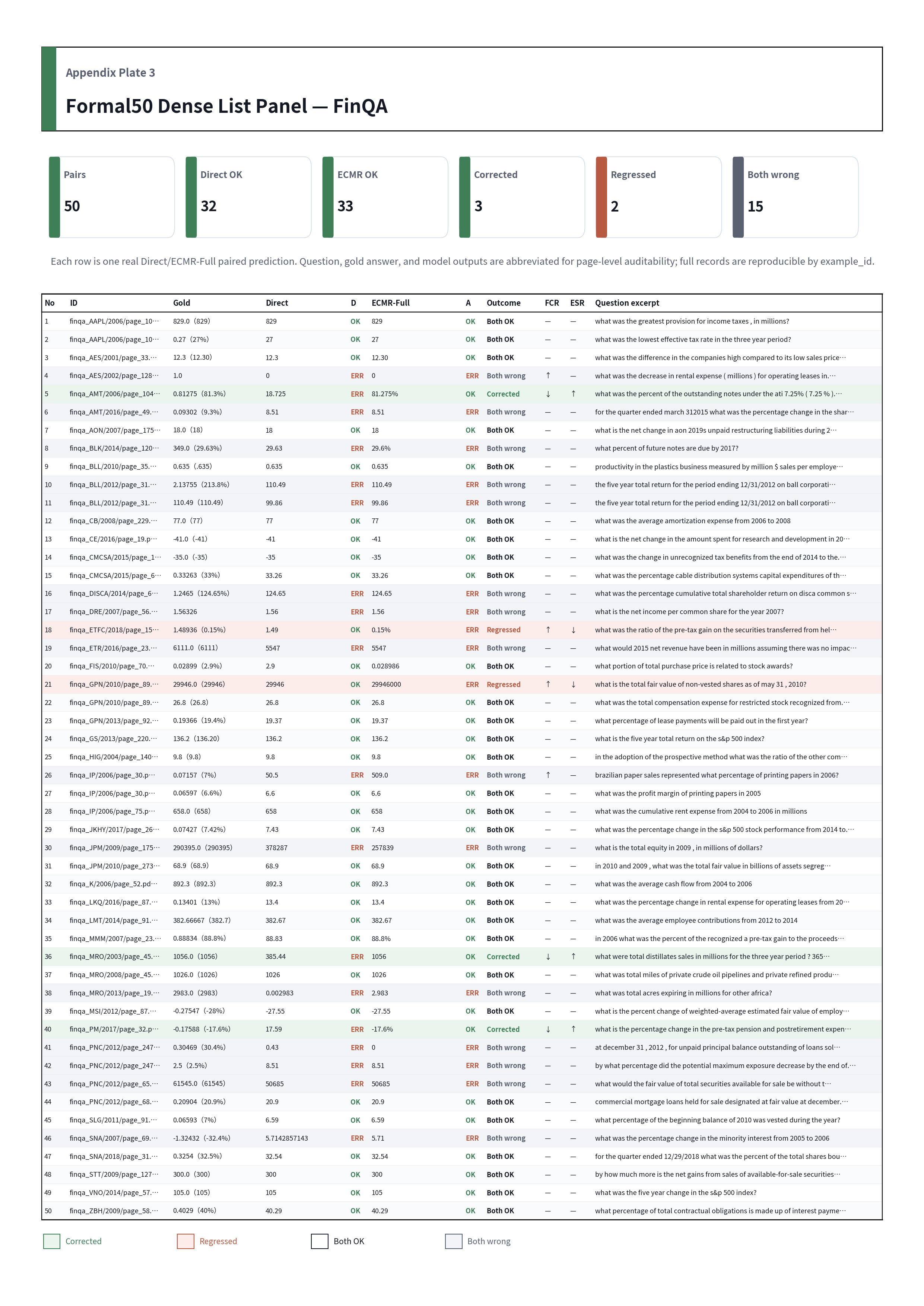}
\end{figure*}
\begin{figure*}[p]
\centering
\includegraphics[width=\textwidth,height=8.8in,keepaspectratio]{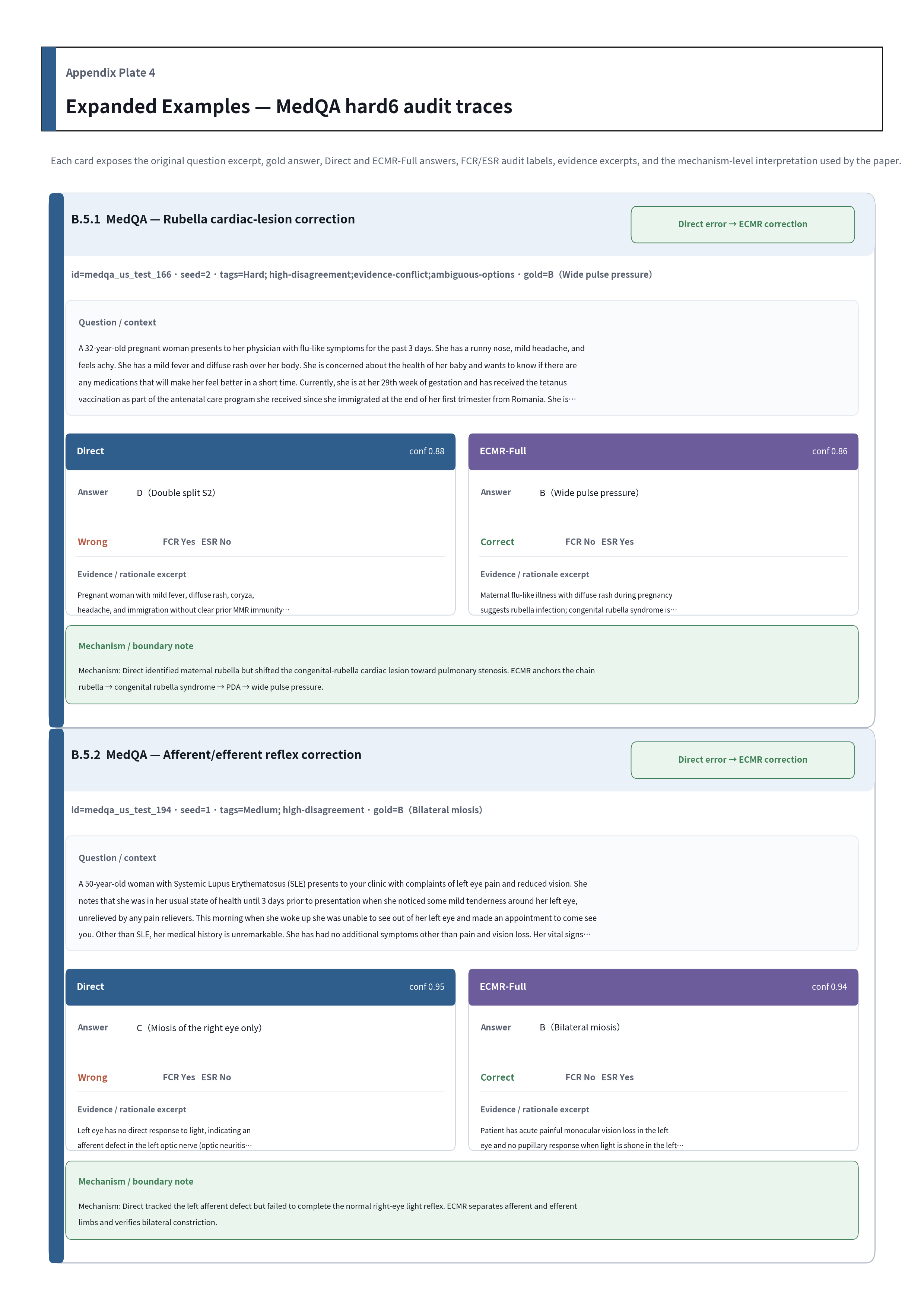}
\end{figure*}
\begin{figure*}[p]
\centering
\includegraphics[width=\textwidth,height=8.8in,keepaspectratio]{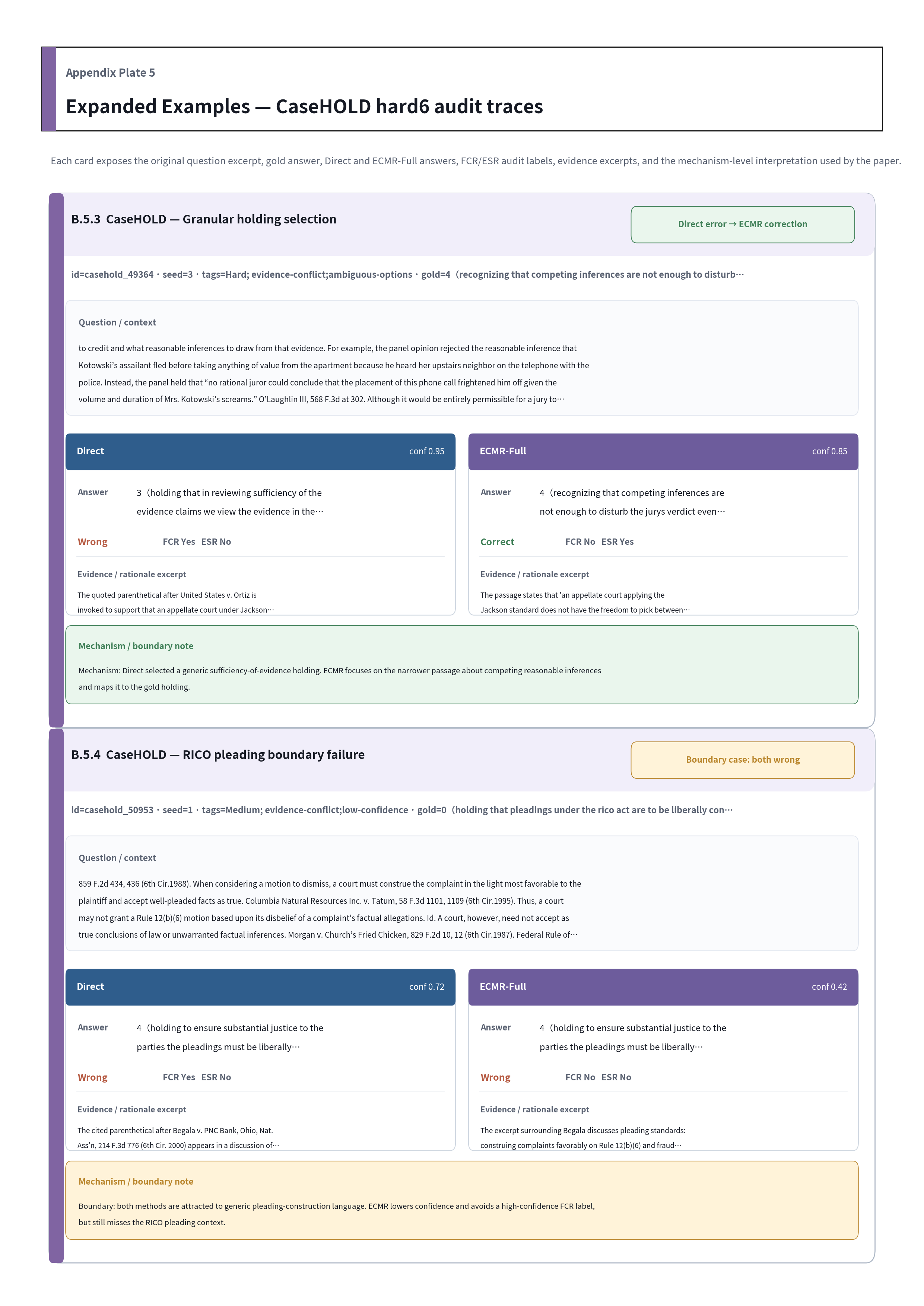}
\end{figure*}
\begin{figure*}[p]
\centering
\includegraphics[width=\textwidth,height=8.8in,keepaspectratio]{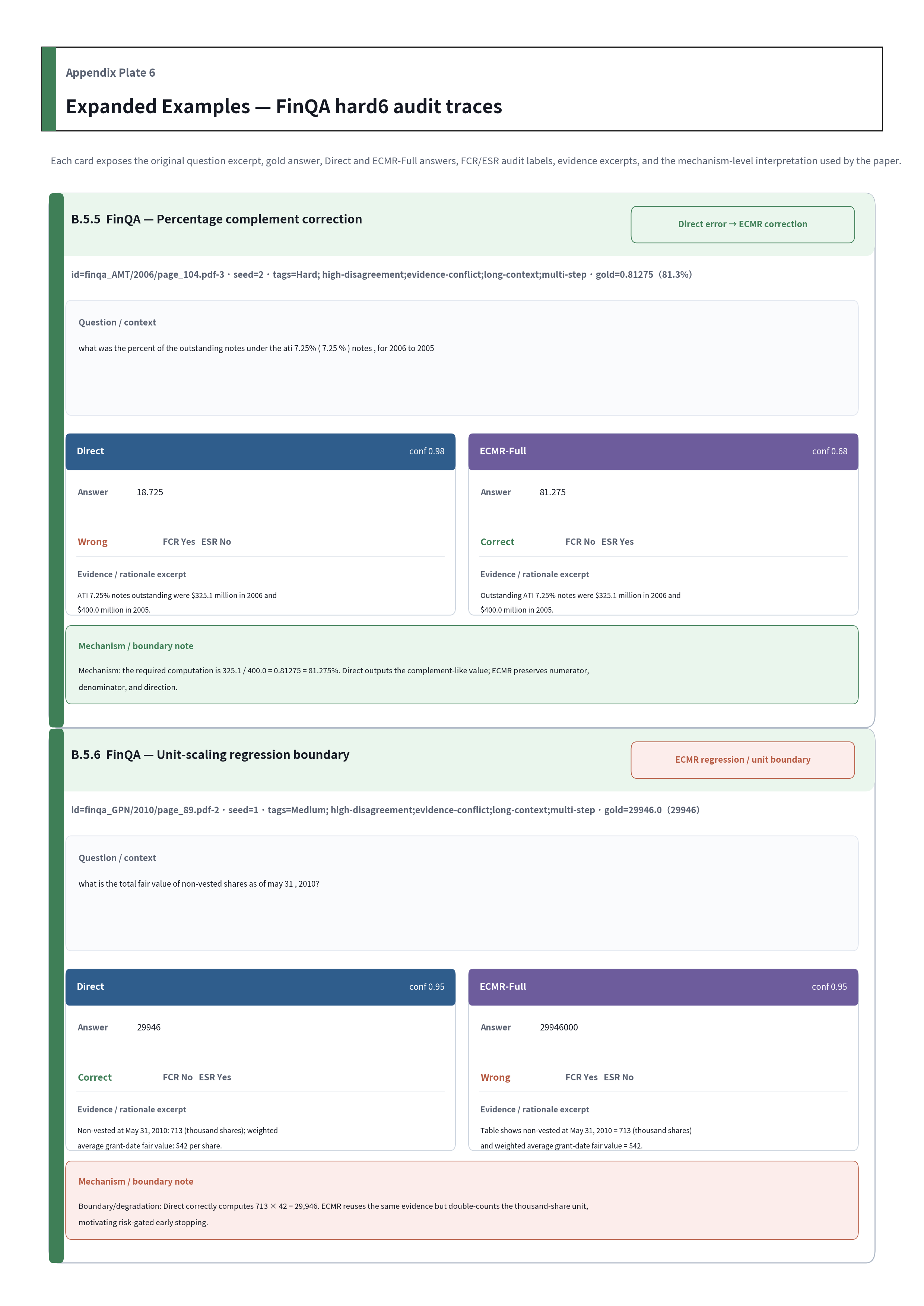}
\end{figure*}


\clearpage
\begin{thebibliography}{25}
\providecommand{\natexlab}[1]{#1}
\providecommand{\url}[1]{\texttt{#1}}
\expandafter\ifx\csname urlstyle\endcsname\relax
  \providecommand{\doi}[1]{doi: #1}\else
  \providecommand{\doi}{doi: \begingroup \urlstyle{rm}\Url}\fi

\bibitem[Chen et~al.(2023{\natexlab{a}})Chen, Dong, Shu, Zhang, Sesay,
  Karlsson, Fu, and Shi]{chen2023autoagents}
Guangyao Chen, Siwei Dong, Yu~Shu, Ge~Zhang, Jaward Sesay, B{\"o}rje~F.
  Karlsson, Jie Fu, and Yemin Shi.
\newblock Autoagents: A framework for automatic agent generation,
  2023{\natexlab{a}}.
\newblock arXiv:2309.17288.

\bibitem[Chen et~al.(2024)Chen, Saha, and Bansal]{chen2024reconcile}
Justin Chen, Swarnadeep Saha, and Mohit Bansal.
\newblock Reconcile: Round-table conference improves reasoning via consensus
  among diverse llms.
\newblock In \emph{Proceedings of the 62nd Annual Meeting of the Association
  for Computational Linguistics (Volume 1: Long Papers)}, 2024.
\newblock \doi{10.18653/v1/2024.acl-long.381}.

\bibitem[Chen et~al.(2023{\natexlab{b}})Chen, Zaharia, and
  Zou]{chen2023frugalgpt}
Lingjiao Chen, Matei Zaharia, and James Zou.
\newblock Frugalgpt: How to use large language models while reducing cost and
  improving performance.
\newblock \emph{Advances in Neural Information Processing Systems},
  2023{\natexlab{b}}.
\newblock \doi{10.48550/arxiv.2305.05176}.

\bibitem[Chen et~al.(2023{\natexlab{c}})Chen, Ma, Wang, and
  Cohen]{chen2022program}
Wenhu Chen, Xueguang Ma, Xinyi Wang, and William~W. Cohen.
\newblock Program of thoughts prompting: Disentangling computation from
  reasoning for numerical reasoning tasks.
\newblock \emph{Transactions on Machine Learning Research}, 2023{\natexlab{c}}.
\newblock \doi{10.48550/arxiv.2211.12588}.
\newblock URL \url{https://openreview.net/forum?id=YfZ4ZPt8zd}.

\bibitem[Chen et~al.(2021)Chen, Chen, Smiley, Shah, Borova, Langdon, Moussa,
  Beane, Huang, Routledge, and Wang]{chen2021finqa}
Zhiyu Chen, Wenhu Chen, Charese Smiley, Sameena Shah, Iana Borova, Dylan
  Langdon, Reema Moussa, Matt Beane, Ting-Hao Huang, Bryan Routledge, and
  William~Yang Wang.
\newblock Finqa: A dataset of numerical reasoning over financial data.
\newblock In \emph{Proceedings of the 2021 Conference on Empirical Methods in
  Natural Language Processing}, 2021.
\newblock \doi{10.18653/v1/2021.emnlp-main.300}.

\bibitem[Du et~al.(2023)Du, Li, Torralba, Tenenbaum, and
  Mordatch]{du2023improving}
Yilun Du, Shuang Li, Antonio Torralba, Joshua~B. Tenenbaum, and Igor Mordatch.
\newblock Improving factuality and reasoning in language models through
  multiagent debate, 2023.

\bibitem[Farquhar et~al.(2024)Farquhar, Kossen, Kuhn, and
  Gal]{farquhar2024detecting}
Sebastian Farquhar, Jannik Kossen, Lorenz Kuhn, and Yarin Gal.
\newblock Detecting hallucinations in large language models using semantic
  entropy.
\newblock \emph{Nature}, 2024.
\newblock \doi{10.1038/s41586-024-07421-0}.

\bibitem[Jiang et~al.(2023)Jiang, Ren, and Lin]{jiang2023llmblender}
Dongfu Jiang, Xiang Ren, and Bill Lin.
\newblock Llm-blender: Ensembling large language models with pairwise ranking
  and generative fusion.
\newblock In \emph{Proceedings of the 61st Annual Meeting of the Association
  for Computational Linguistics (Volume 1: Long Papers)}, 2023.
\newblock \doi{10.18653/v1/2023.acl-long.792}.

\bibitem[Jin et~al.(2021)Jin, Pan, Oufattole, Weng, Fang, and
  Szolovits]{jin2021what}
Di~Jin, Eileen Pan, Nassim Oufattole, Wei-Hung Weng, Hanyi Fang, and Peter
  Szolovits.
\newblock What disease does this patient have? a large-scale open domain
  question answering dataset from medical exams.
\newblock \emph{Applied Sciences}, 2021.
\newblock \doi{10.3390/app11146421}.

\bibitem[Li et~al.(2024)Li, Du, Zhang, Hou, Grabowski, Li, and
  Ie]{li2024improving}
Yunxuan Li, Yibing Du, Jiageng Zhang, Longxiang Hou, Peter Grabowski, Yeqing
  Li, and Eugene Ie.
\newblock Improving multi-agent debate with sparse communication topology.
\newblock In \emph{Findings of the Association for Computational Linguistics:
  EMNLP 2024}, 2024.
\newblock \doi{10.18653/v1/2024.findings-emnlp.427}.

\bibitem[Liang et~al.(2024)Liang, He, Jiao, Wang, Wang, Wang, Yang, Shi, and
  Tu]{tian2024encouraging}
Tian Liang, Zhiwei He, Wenxiang Jiao, Xing Wang, Yan Wang, Rui Wang, Yujiu
  Yang, Shuming Shi, and Zhaopeng Tu.
\newblock Encouraging divergent thinking in large language models through
  multi-agent debate.
\newblock In \emph{Proceedings of the 2024 Conference on Empirical Methods in
  Natural Language Processing}, 2024.
\newblock \doi{10.18653/v1/2024.emnlp-main.992}.

\bibitem[Liu et~al.(2023)Liu, Iter, Xu, Wang, Xu, and Zhu]{liu2023geval}
Yang Liu, Dan Iter, Yichong Xu, Shuohang Wang, Ruochen Xu, and Chenguang Zhu.
\newblock G-eval: Nlg evaluation using gpt-4 with better human alignment.
\newblock In \emph{Proceedings of the 2023 Conference on Empirical Methods in
  Natural Language Processing}, 2023.
\newblock \doi{10.18653/v1/2023.emnlp-main.153}.

\bibitem[Madaan et~al.(2023)Madaan, Tandon, Gupta, Hallinan, Gao, Wiegreffe,
  Alon, Dziri, Prabhumoye, Yang, Gupta, Majumder, Hermann, Welleck,
  Yazdanbakhsh, and Clark]{madaan2023selfrefine}
Aman Madaan, Niket Tandon, Prakhar Gupta, Skyler Hallinan, Luyu Gao, Sarah
  Wiegreffe, Uri Alon, Nouha Dziri, Shrimai Prabhumoye, Yiming Yang, Shashank
  Gupta, Bodhisattwa~Prasad Majumder, Katherine Hermann, Sean Welleck, Amir
  Yazdanbakhsh, and Peter Clark.
\newblock Self-refine: Iterative refinement with self-feedback.
\newblock \emph{Advances in Neural Information Processing Systems}, 2023.
\newblock \doi{10.48550/arxiv.2303.17651}.

\bibitem[Qian et~al.(2024)Qian, Liu, Liu, Chen, Dang, Li, Yang, Chen, Su, Cong,
  Xu, Li, Liu, and Sun]{qian2024chatdev}
Chen Qian, Wei Liu, Hongzhang Liu, Nuo Chen, Yufan Dang, Jiahao Li, Cheng Yang,
  Weize Chen, Yu-Sheng Su, Xin Cong, Juyuan Xu, Dahai Li, Zhiyuan Liu, and
  Maosong Sun.
\newblock Chatdev: Communicative agents for software development.
\newblock In \emph{Proceedings of the 62nd Annual Meeting of the Association
  for Computational Linguistics (Volume 1: Long Papers)}, 2024.
\newblock \doi{10.18653/v1/2024.acl-long.810}.

\bibitem[Schick et~al.(2023)Schick, Dwivedi-Yu, Dessi, Raileanu, Lomeli,
  Hambro, Zettlemoyer, Cancedda, and Scialom]{schick2023toolformer}
Timo Schick, Jane Dwivedi-Yu, Roberto Dessi, Roberta Raileanu, Maria Lomeli,
  Eric Hambro, Luke Zettlemoyer, Nicola Cancedda, and Thomas Scialom.
\newblock Toolformer: Language models can teach themselves to use tools.
\newblock \emph{Advances in Neural Information Processing Systems}, 2023.
\newblock \doi{10.48550/arxiv.2302.04761}.

\bibitem[Shinn et~al.(2023)Shinn, Cassano, Gopinath, Narasimhan, and
  Yao]{shinn2023reflexion}
Noah Shinn, Federico Cassano, Ashwin Gopinath, Karthik Narasimhan, and Shunyu
  Yao.
\newblock Reflexion: Language agents with verbal reinforcement learning.
\newblock \emph{Advances in Neural Information Processing Systems}, 2023.
\newblock \doi{10.48550/arxiv.2303.11366}.

\bibitem[Tang et~al.(2024)Tang, Zou, Zhang, Li, Zhao, Zhang, Cohan, and
  Gerstein]{tang2024medagents}
Xiangru Tang, Anni Zou, Zhuosheng Zhang, Ziming Li, Yilun Zhao, Xingyao Zhang,
  Arman Cohan, and Mark Gerstein.
\newblock Medagents: Large language models as collaborators for zero-shot
  medical reasoning.
\newblock In \emph{Findings of the Association for Computational Linguistics:
  ACL 2024}, 2024.
\newblock \doi{10.18653/v1/2024.findings-acl.33}.

\bibitem[Wang et~al.(2023)Wang, Yue, and Sun]{wang2023can}
Boshi Wang, Xiang Yue, and Huan Sun.
\newblock Can chatgpt defend its belief in truth? evaluating llm reasoning via
  debate.
\newblock In \emph{Findings of the Association for Computational Linguistics:
  EMNLP 2023}, 2023.
\newblock \doi{10.18653/v1/2023.findings-emnlp.795}.

\bibitem[Wang et~al.(2024)Wang, Ma, Feng, Zhang, Yang, Zhang, Chen, Tang, Chen,
  Lin, Zhao, Wei, and Wen]{wang2024survey}
Lei Wang, Chen Ma, Xueyang Feng, Zeyu Zhang, Hao Yang, Jingsen Zhang, Zhiyuan
  Chen, Jiakai Tang, Xu~Chen, Yankai Lin, Wayne~Xin Zhao, Zhewei Wei, and
  Ji-Rong Wen.
\newblock A survey on large language model based autonomous agents.
\newblock \emph{Frontiers of Computer Science}, 2024.
\newblock \doi{10.1007/s11704-024-40231-1}.

\bibitem[Wang et~al.(2022)Wang, Wei, Schuurmans, Le, H., Narang, Chowdhery, and
  Zhou]{wang2022selfconsistency}
Xuezhi Wang, Jason Wei, Dale Schuurmans, Quoc~V. Le, Ed~H., Sharan Narang,
  Aakanksha Chowdhery, and Denny Zhou.
\newblock Self-consistency improves chain of thought reasoning in language
  models.
\newblock \emph{International Conference on Learning Representations}, 2022.
\newblock \doi{10.48550/arxiv.2203.11171}.

\bibitem[Wei et~al.(2022)Wei, Wang, Schuurmans, Bosma, Ichter, Xia, Chi, Le,
  and Zhou]{wei2022chainofthought}
Jason Wei, Xuezhi Wang, Dale Schuurmans, Maarten Bosma, Brian Ichter, Fei Xia,
  Ed~Chi, Quoc Le, and Denny Zhou.
\newblock Chain-of-thought prompting elicits reasoning in large language
  models.
\newblock \emph{Advances in Neural Information Processing Systems}, 2022.
\newblock \doi{10.48550/arxiv.2201.11903}.

\bibitem[Wu et~al.(2023)Wu, Bansal, Zhang, Wu, Li, Zhu, Jiang, Zhang, Zhang,
  Liu, Awadallah, White, Burger, and Wang]{wu2023autogen}
Qingyun Wu, Gagan Bansal, Jieyu Zhang, Yiran Wu, Beibin Li, Erkang Zhu,
  Li~Jiang, Xiaoyun Zhang, Shaokun Zhang, Jiale Liu, Ahmed~Hassan Awadallah,
  Ryen~W. White, Doug Burger, and Chi Wang.
\newblock Autogen: Enabling next-gen llm applications via multi-agent
  conversation, 2023.

\bibitem[Yao et~al.(2023{\natexlab{a}})Yao, Yu, Zhao, Shafran, Griffiths, Cao,
  and Narasimhan]{yao2023tree}
Shunyu Yao, Dian Yu, Jeffrey Zhao, Izhak Shafran, Tom Griffiths, Yuan Cao, and
  Karthik Narasimhan.
\newblock Tree of thoughts: Deliberate problem solving with large language
  models.
\newblock \emph{Advances in Neural Information Processing Systems},
  2023{\natexlab{a}}.
\newblock \doi{10.48550/arxiv.2305.10601}.

\bibitem[Yao et~al.(2023{\natexlab{b}})Yao, Zhao, Yu, Du, Shafran, Narasimhan,
  and Cao]{yao2023react}
Shunyu Yao, Jeffrey Zhao, Dian Yu, Nan Du, Izhak Shafran, Karthik Narasimhan,
  and Yuan Cao.
\newblock React: Synergizing reasoning and acting in language models.
\newblock In \emph{International Conference on Learning Representations},
  2023{\natexlab{b}}.
\newblock \doi{10.48550/arxiv.2210.03629}.

\bibitem[Zheng et~al.(2021)Zheng, Guha, Anderson, Henderson, and
  Ho]{zheng2021casehold}
Lucia Zheng, Neel Guha, Brandon~R. Anderson, Peter Henderson, and Daniel~E. Ho.
\newblock When does pretraining help? assessing self-supervised learning for
  law and the casehold dataset, 2021.

\end{thebibliography}
\end{document}